\documentclass{article}

\usepackage[utf8]{inputenc} 
\usepackage[USenglish]{babel}
\usepackage{csquotes}
\usepackage[T1]{fontenc}    

\usepackage[nonatbib,preprint]{neurips_2025}

\usepackage{mathtools} 
\usepackage{amssymb}            
\usepackage{amsthm}

\usepackage{algorithm}
\usepackage{algorithmic}
\usepackage{subfigure}
\usepackage{amsmath,bm}
\usepackage{thmtools}

\newtheorem{theorem}{Theorem}
\newtheorem{lemma}{Lemma}
\newtheorem*{theorem*}{Theorem}
\newtheorem*{lemma*}{Lemma}
\newtheorem{fact}{Fact}

\newcommand{\fakelabel}[2]{%
    \refstepcounter{#1}%
    \label{#2}%
    \addtocounter{#1}{-1}}

\usepackage{hyperref}       
\usepackage{url}            
\usepackage{booktabs}       
\usepackage{amsfonts}       
\usepackage{nicefrac}       
\usepackage{microtype}      
\usepackage{xcolor}         

\usepackage[style=numeric-comp,natbib=true,uniquelist=true,maxbibnames=10,maxcitenames=2,sorting=none]{biblatex}
\renewbibmacro{in:}{}
\DeclareNameAlias{author}{given-family}
\DeclareMathOperator*{\E}{\mathbb{E}}
\DeclareMathOperator{\Ber}{Ber}
\DeclareMathOperator{\SG}{SG}
\renewcommand{\det}{\operatorname{\mathbf{det}}}

\usepackage[draft,layout={margin}]{fixme}

\title{Efficient kernelized bandit algorithms\\via exploration distributions}

\author{%
  Bingshan Hu\\
  University of British Columbia\\
  \texttt{bingsha1@cs.ubc.ca}
  \And
  Zheng He\\
  University of British Columbia\\
  \texttt{zhhe@cs.ubc.ca}
  \And
  Danica J.~Sutherland\\
  UBC \& Amii\\
  \texttt{dsuth@cs.ubc.ca}
}

\begin{document}

\maketitle

\begin{abstract}
 We consider a kernelized bandit problem with a compact arm set $\mathcal{X} \subset \mathbb{R}^d $ and a fixed but unknown reward function $f^*$ with a finite norm in some Reproducing Kernel Hilbert Space (RKHS). We propose a class of computationally efficient kernelized bandit algorithms, which we call GP-Generic, based on a novel concept: exploration distributions. 
 This class of algorithms includes Upper Confidence Bound-based approaches as a special case, but also allows for a variety of randomized algorithms.
With careful choice of exploration distribution, our proposed  generic algorithm  realizes a wide range of concrete algorithms that achieve $\tilde{O}(\gamma_T\sqrt{T})$ regret bounds, where $\gamma_T$ characterizes the RKHS complexity. This matches known results for UCB- and Thompson Sampling-based algorithms;
we also show that in practice, randomization can yield better practical results.

%
\end{abstract}




\section{Introduction} \label{sec: intro}

This work studies kernelized bandits \citep{srinivas2010gaussian,chowdhury2017kernelized,whitehouse2023sublinear}. In this learning problem we have a compact arm set $\mathcal{X} \subset \mathbb{R}^d$ and  a fixed but unknown reward function $f^*: \mathcal{X} \rightarrow \mathbb{R}$  living in some Reproducing Kernel Hilbert Space (RKHS) with a finite norm.
In each round $t$, a
learning agent pulls an arm $X_t \in \mathcal{X}$ and obtains a random reward $Y_t = f^*(X_t) + \varepsilon_t$, where $\varepsilon_t$ is a zero-mean sub-Gaussian random variable.  The goal of the learning agent is to
pull arms sequentially to accumulate as much reward as
possible over $T$ rounds, for some finite $T$.
Since the learning agent may not always pull the optimal arm,
we measure the agent's performance based on
regret:
the performance gap between the highest achievable reward  and the actually earned reward.
Bandit problems have found a wide variety of applications across many domains \citep{bandit-apps};
the flexibility of the kernelized setting
allows modeling complex correlation structures among potentially infinitely many arms,
which additionally aids application in areas such as Bayesian optimization \citep{garnett_bayesoptbook_2023}.
An alternative, but closely related, viewpoint on this problem is via the framework of
Gaussian processes (GPs) \citep{seeger2004gaussian}.

Intuitively, to achieve low regret, a bandit algorithm should use information about the past to make future decisions.
Specifically, 
the agent must dynamically negotiate between two competing goals: exploitation, which prioritizes pulling arms to optimize immediate performance, and exploration, which emphasizes gathering information to systematically reduce  uncertainty and expand the probabilistic understanding of the unknown reward function.

Both theoretically and empirically, the deterministic Upper Confidence Bound (UCB) algorithms \citep{auer2002finite,chowdhury2017kernelized,whitehouse2023sublinear,srinivas2010gaussian} and randomized Thompson Sampling (TS)-inspired algorithms \citep{agrawalnear,chowdhury2017kernelized,agrawal2013thompson} 
have been shown to be efficient at trading off exploitation and exploration.
We can think of both categories of algorithms as drawing a choice of arm from a data-dependent distribution:
in UCB, that distribution is a degenerate point mass
at the arm with highest upper confidence bound,
while in TS
it is the data-dependent
distribution
of the algorithm's current belief that the given arm is optimal,
which typically has some positive probability to select any candidate arm.
These choices of distributions,
when set up appropriately,
do indeed trade off between exploration and exploitation.
 
In kernelized bandits, drawing arms to pull from data-dependent distributions  generates  a sequence of dependent random variables $Y_1, Y_2, \dotsc, Y_T$, which prevents the direct application of traditional sub-Gaussian concentration bounds 
\citep{chu2011contextual,valko2013finite,abbasi2011improved}.
To tackle this challenge, previous analyses \citep{abbasi2011improved,chowdhury2017kernelized,whitehouse2023sublinear,agrawal2013thompson} 
identify martingale structures embedded in the gathered data trajectory  $(X_1, Y_1), (X_1, Y_2), \dotsc, (X_T, Y_T)$, and then apply concentration bounds in a \emph{high-dimensional vector-valued space}.
 The state-of-the-art algorithms of this type achieve  $\tilde{O} (\gamma_T \sqrt{T})$ regret bounds, where $\gamma_T$ characterizes the complexity of  the RKHS. 
These bounds are sub-optimal;
the regret lower bound of
$\Omega \left(\sqrt{T \gamma_T } \right)$
\citep{cai2025lower,scarlett2017lower}
is known to be tight
\citep{valko2013finite,salgia2024random,li2022gaussian}.
{This sub-optimality comes from
using concentration bounds over  dependent random variables $Y_1, Y_2, \dotsc, Y_T$ in a high-dimensional vector space.}

In addition to the UCB and TS, Arm Elimination-based  algorithms \citep{auer2010ucb,valko2013finite,salgia2024random,auer2002using,li2022gaussian} are also able to trade off exploitation and exploration. 
This category of algorithms conducts learning based on epochs, and identifies a non-increasing sequence of arm subsets $\mathcal{X}_r \subseteq \mathcal{X}_{r-1}\subseteq \dotsc \subseteq \mathcal{X}_1 = \mathcal{X}$ that balances exploitation and exploration  at the end of each epoch $r$. Given $\mathcal{X}_r$, in each round belonging to epoch $r$, arm elimination algorithms draw an arm to pull   according to 
 \emph{data-independent distributions}. For example,  Random Exploration with Domain Shrinking (REDS) \citep{salgia2024random} uses a uniform distribution, and SupKernelUCB \citep{valko2013finite} and Batched Pure Exploration (BPE) \citep{li2022gaussian}  use  deterministic distributions independent of the data drawn in that epoch.
 This makes the generated reward sequence within that epoch 
 \emph{independent} random variables, enabling the application of traditional concentration bounds in a \emph{one-dimensional real space}, rather than a high-dimensional vector space. SupKernelUCB, REDS, and BPE each achieve the optimal $\tilde{O}\left(\sqrt{T\gamma_T } \right)$ regret bounds. 
 Even so, they often have poor empirical performance
 and can require heavy computation to discretize $\mathcal X$ into $\mathcal X_r$,
 which often prevents their practical deployment compared to UCB-based and TS-based algorithms.

This work focuses on algorithms which use data-dependent distributions to draw an arm to pull in each round,
rather than epoch-based arm elimination.
Since UCB-based algorithms and TS-based algorithms can both be viewed as concrete algorithms in this framework, it is natural to ask whether there exist intrinsic intersections behind these two families of algorithms, what the common features shared by these two families of algorithms are, and whether the theoretical analysis can be unified.

We step towards answering these questions by introducing a novel concept:
exploration distributions.
These are distributions on the reals,
for instance standard normal or a point mass at $1$,
which control the exploration-exploitation tradeoff in each round. 
The chosen exploration distribution facilitates the learning agent's to explore the unknown environment to reduce uncertainty in that round; each sequence of exploration distributions over $T$ rounds realizes a concrete algorithm. Based on the concept of exploration distribution, we propose a generic algorithm, Generic-GP, which is able to   realize a wide range of algorithms depending on the chosen sequence of exploration distributions. 
Choosing a degenerate constant distribution yields a UCB variant,
while other distributions insert a controlled amount of randomness into the process,
resembling TS-based schemes.

\textbf{Preview of exploration distributions.}
We say a sequence of distributions $P_{w,1}, P_{w,2}, \dotsc, P_{w,t}$ over $\mathbb R$ can trade off exploitation-exploration  efficiently, i.e., yields an algorithm with a sub-linear regret bound, if  this sequence of distributions satisfies the following conditions simultaneously: 
 \begin{equation}
        \begin{array}{c}
            \mathbb{P}_{w_s \sim P_{w,s}} \left\{w_s \ge 1 \right\} = C_{1,s} \ne 0,\quad \forall s = 1,2, \dotsc, t, \\
         \E \left[\mathop{\max}_{s \in [t]}|w_{s}| \right] = C_{2,t},
         \qquad
         \left(\E \left[\mathop{\max}_{s \in [t]     } \frac{|w_s|}{C_{1,s}} \right] + \mathop{\max}\limits_{s \in [t]     } \E \left[\frac{|w_s|}{C_{1,s}} \right]\right) = C_{3,t},
        \end{array}
    \end{equation}
 where $w_1, w_2, \dotsc, w_t$ are independent random variables. It is important to note that $P_{w,t}$ can be chosen depending on the past information,
 and can itself be any distribution satisfying the above conditions.
 Each sequence of exploration distributions realize a bandit algorithm. 



\textbf{Preview of Theorem~\ref{theorem: expected regret}.}   The  regret  of the algorithm using $P_{w,1}, \dotsc, P_{w,T}$ for exploration is   
    \begin{equation}
        \begin{array}{l}
       O \left( (C_{2,T} + C_{3,T}) \gamma_T \sqrt{T} 
             + \sqrt{T \gamma_T \ln(T)}   \right)\quad.      
        \end{array}
    \end{equation}
    If $C_{2,T}, C_{3,T}$ are $\mathcal O(1)$, as is the case for some exploration distributions, our proposed generic algorithm achieves the state-of-the-art $O(\gamma_T \sqrt{T})$   bounds 
    shared by UCB and TS-based algorithms.
    

\textbf{Preview of concrete algorithms.} We present some concrete realized algorithms by using simple exploration distributions. 
 (1)~\textbf{Simple-UCB} uses a degenerate exploration distribution which always evaluates to the number $1$, i.e.\ a Bernoulli distribution with success probability $1$, in all rounds. It achieves $ \tilde{O} (\gamma_T \sqrt{T} 
             )$ regret and can be viewed a deterministic UCB-based algorithm.
(2)~\textbf{Simple-Bernoulli}  uses Bernoulli distributions with success probabilities $p_1, p_2, \dotsc, p_T \in (0,1]$ as exploration distributions;
that is, with probability $p_t$ it runs a step of \textbf{Simple-UCB},
while with probability $1-p_t$ it does pure exploitation.
This algorithm achieves $\tilde{O} ( \gamma_T\sqrt{T}/\mathop{\min}_{t \in [T]} p_t )$ regret.
(3)~\textbf{Simple-Gaussian} uses the standard Gaussian distribution, $ \mathcal{N}(0,1)$,  in all rounds,  yielding $\tilde{O} (\gamma_T\sqrt{T}      )$ regret.

 



\section{Preliminaries}


\paragraph{Bandit learning.} \label{sec: bandit background}
We consider a kernelized bandit problem with a compact arm set $\mathcal{X} \subset \mathbb{R}^d $ and a fixed but unknown reward function $f^* : \mathcal{X} \rightarrow \mathbb{R}$
over the arms. The arm set may be infinite.
The learning protocol is that in each round $t$, Learner chooses an arm $X_t \in \mathcal
{X}$ and observes a noisy reward $Y_t := f^*(X_t) +\varepsilon_t$, where $\varepsilon_t$ is the feedback noise in round $t$. 
The noise random variable $\varepsilon_t$ is a conditionally zero-mean $R$-sub-Gaussian random variable, i.e., we have
\begin{equation}\label{noise conditional}
    \forall \lambda \in \mathbb{R}, \quad \E \left[e^{\lambda \varepsilon_t} \mid X_t, \mathcal{F}_{t-1} \right] \le e^{\frac{1}{2} \lambda^2 R^2} ,
\end{equation}
where $\mathcal{F}_t := \left\{X_1, Y_1, X_2, Y_2, \dotsc, X_t, Y_t \right\} $ collects all the history information by the end of round $t$: the pulled arms and the observed rewards. We set $\mathcal{F}_0 = \{ \}$.


Let $x_* = \mathop{\arg\max}_{x \in \mathcal{X}} f^*(x)$ denote the unique optimal arm. The goal of Learner is to choose arms sequentially to minimize the (pseudo-)regret, defined as
\begin{equation}\label{def: regret}
         \mathcal{R}(T)  = \sum\limits_{t=1}^{T}f^*(x_*) - \E \left[f^*(X_t) \right] ,
\end{equation}
where the expectation is taken over the arm sequence $X_1, \dotsc, X_t, \dotsc X_T$. Regret measures the expected cumulative performance gap between always pulling the optimal arm and the actual sequence of arms pulled by Learner.


\paragraph{RKHS.}
Learner would like to learn $f^*$ by choosing a finite sequence of  points $X_1, X_2, \dotsc, X_T \in \mathcal{X}$. However, if Learner is only allowed to query a finite number of  points,  learning a totally arbitrary $f^*$ is hopeless. To make the learning problem interesting, we need to set some constraints on $f^*$: that it satisfies some kind of smoothness property.

In kernelized bandits, we assume the unknown reward function $f^*$ lives in a reproducing kernel Hilbert space (an RKHS)
with a bounded kernel $\sup_{x \in \mathcal X} K(x,x) \le 1$,
and moreover that this function has RKHS norm bounded by $\left\lVert  f \right\rVert \le  D$.
Specifically,
let $K: \mathcal{X} \times \mathcal{X} \rightarrow \mathbb{R}$ be a positive definite kernel. 
A Hilbert space $H_K$ of functions $\mathcal{X} \to \mathbb R$ associated with an inner product $\left\langle \cdot, \cdot \right\rangle$ is called an RKHS with reproducing kernel $K$ if $K(\cdot, x) \in H_K$  and $f(x) =\left\langle f, K(\cdot, x) \right\rangle$ for all $x \in \mathcal{X}$.
The inner product induces the RKHS norm, denoted $\left\lVert f \right\rVert = \sqrt{ \left\langle f, f \right\rangle}$,
which gives a notion of smoothness of $f$ with respect to $K$.


Let $V_t:= \sum_{s = 1}^{t} K(\cdot, X_{s}) K(\cdot, X_{s})^{\top}   $  and 
$\overline{V}_t :=  I_{H} + V_{t} $, where $I_H$ is the identity operator on $H_k$,
and $f g^\top$ denotes the outer product
$f g^\top : H_K \to H_K$ given by $(f g^\top) h = f \langle g, h \rangle$.
Both $V_t $ and $\overline{V}_t$ are determined by $\mathcal{F}_t$,
since they depend only on the sequence
$K(\cdot, X_{1}),K(\cdot, X_{2}), \dotsc, K(\cdot, X_{t})$.


\paragraph{RKHS complexity.} The maximum information gain $\gamma_t(\rho)$ characterizes the effective dimension of the kernel, where $\rho >0$ is  the regularization parameter. From \citet{whitehouse2023sublinear}, we have
\begin{equation}
     \gamma_t(\rho)  := \sup_{x_1, x_2, \dotsc, x_t \in \mathcal{X}} \frac{1}{2} \log \left(\det \left(I_H + \rho^{-1} V_t \right) \right) ,
\label{def: RKHS complexity}
    \end{equation}
    where $\det(A)$ is the determinant of a matrix $A$. For brevity, we use  $\gamma_t$ to denote the maximum information gain using regularization parameter $\rho =1$.
    
For exmple, the commonly used Matérn kernel has $\gamma_T = O (T^{\frac{d}{2\nu+d}}\log^{\frac{2\nu}{2\nu+d}}(T) )$ \citep{vakili2021information}, where $\nu >0$ is the smoothness parameter.    The squared exponential (SE) kernel  has $\gamma_T = O (\log^{d+1}(T) )$ \citep{vakili2021information}. The linear kernel has $\gamma_T = O ( d \log(1+ T/d))$ \citep{abbasi2011improved}.
Our proposed generic  algorithm can also be applied to other kernels with  polynomially- or exponentially-decaying eigenvalues \citep{vakili2021information}.



\section{Related Work} \label{sec: literature}
There is a large literature  on kernelized  bandits \citep[e.g.][]{valko2013finite,chowdhury2017kernelized,janz2020bandit,whitehouse2023sublinear,salgia2024random,scarlett2017lower,cai2025lower,shekhar2018gaussian,shekhar2022instance,abbasi2011improved,srinivas2010gaussian,lattimore2023lower,vakili2021information,li2022gaussian,vakili2022open,vakili2021optimal,vakili2021open,vakili2024open,ray2019bayesian,lee2022multi,durand2018streaming,flynn2024tighter,kim2024enhancing}. 
Algorithms enjoying sub-linear regret bounds, i.e., satisfying $\mathcal{R}(T)/T \to 0$ as $T \rightarrow + \infty$, typically need to use history information to make decisions: past evidence is able to  help Learner to make balanced  decisions between exploitation and exploration.
As opposed to $N$-armed bandits where we can bound the expected number of pulls of each  sub-optimal arm separately\footnote{In an $N$-armed bandit problem, all rewards associated with any fixed arm are i.i.d.\ according to a fixed (unknown) reward distribution. We can thus use standard i.i.d.\ sub-Gaussian bounds.}, kernelized bandits have their own challenges in terms of using the gathered history information. 
They must use the history information somehow to achieve sublinear regret.
But as previously mentioned, algorithms using history information may generate   a sequence  of dependent random rewards $Y_1, Y_2,\dotsc, Y_T$, 
precluding the direct use of traditional sub-Gaussian concentration bounds  \citep{auer2002using,abbasi2011improved,chu2011contextual,valko2013finite}.

 We categorize kernelized bandit algorithms based on whether the generated reward sequence $Y_1, Y_2, \dotsc, Y_T$ is dependent or not, and how to deal with the reward sequence $Y_1, Y_2, \dotsc, Y_T$. As discussed in Section~\ref{sec: intro}, the reward sequence generated by  UCB and TS-based algorithms are dependent, whereas the computationally expensive arm elimination algorithms can generate independent reward sequences in each epoch. \citet{scarlett2017lower,cai2025lower} establish $\Omega (\sqrt{T\gamma_T } )$ regret lower bounds for some kernel, e.g., the Matérn kernel.

The key idea behind UCB-based kernelized bandit algorithms \citep[e.g.][]{chowdhury2017kernelized,janz2020bandit,whitehouse2023sublinear,shekhar2018gaussian,shekhar2022instance,abbasi2011improved,srinivas2010gaussian,vakili2021information} is to construct upper confidence bounds, by 
adding \emph{deterministic exploration bonuses} to the empirical estimates of the underlying unknown reward function. In each round, UCB pulls the arm with the highest upper confidence bound. An alternative view of UCB is to draw an arm to pull from a deterministic data-dependent  distribution, i.e., a distribution that puts all the mass on the UCB maximizer. 
  The best UCB-based algorithms  achieve $\tilde{O}(\gamma_T\sqrt{T})$ regret bounds, which matches the $\tilde{\Omega}(\sqrt{T \gamma_T})$ lower bound \citep{scarlett2017lower}
up to a $\sqrt{\gamma_T}$ factor. 
The usage of concentration bounds in the vector space (e.g., Theorem~1 of \citet{abbasi2011improved}, Theorem~1 of \citet{whitehouse2023sublinear}, or Theorem~1 of \citet{chowdhury2017kernelized}) results in the extra $\sqrt{\gamma_T}$ factors.

The key idea behind TS-based algorithms \citep{chowdhury2017kernelized,agrawal2013thompson} is to use a sequence of data-dependent distributions to model the unknown reward function. In each round, TS  draws a random function from the data-dependent distribution, which can be expressed as adding \emph{random exploration bonuses} (which may be positive or negative) to the empirical estimates of the underlying unknown reward function. TS
pulls the arm that is best under the sampled functional model.
An alternative view of TS is to draw an arm to pull from the data-dependent distribution of being optimal. \citet{chowdhury2017kernelized} presents GP-TS, a TS-based algorithm that uses a sequence of data-dependent multivariate Gaussian distributions to model the unknown reward function.   GP-TS achieves $\tilde{O}(\gamma_T\sqrt{dT} )$ regret bounds, which is an extra $\sqrt{d}$ factor worse than the UCB-based algorithms. The extra $\sqrt{d}$ factor is from the usage of multivariate Gaussian distributions.

Arm elimination-based algorithms, SupKernelUCB of \citet{valko2013finite}, BEF of \citet{li2022gaussian} and REDS of \citet{salgia2024random}, all achieve the optimal $\tilde{O}(\sqrt{T \gamma_T } )$ regret bounds. The key idea behind these optimal algorithms is to conduct  learning based on epochs and at the end of each epoch $r$, only arms $\mathcal{X}_r \subseteq \mathcal{X}_{r-1}$ that empirically perform well in epoch $r$ will be used in the next epoch $r+1$. That is, we maintain a sequence of non-increasing arm subsets $\mathcal{X}_{r+1} \subseteq \mathcal{X}_r \subseteq \mathcal{X}_{r-1} \subseteq \dotsc \subseteq \mathcal{X}_1 = \mathcal{X}$.
Arms whose upper confidence bound
are worse than the lower confidence bound of the empirically best-performing arm in epoch $r$
are eliminated from $\mathcal X_{r+1}$.
In the rounds within epoch $r+1$,  each arm in $\mathcal{X}_{r+1}$ is pulled in a round robin fashion \citep{valko2013finite}, or sampled  according to a data-independent distribution, e.g., a uniform distribution \citep{salgia2024random}. The main downside of these elimination style algorithms is the fine discretization needed for $\mathcal{X}_t$, which is  computationally heavy. 

\section{Generic-GP Algorithm} 

In Section~\ref{sec: exploration distributions}, we present the conditions needed to be used as exploration distributions. In Section~\ref{sec: Generic algorithm},  we present our Generic-GP algorithm. In Section~\ref{sec: analysis}, we present theoretical analysis for our proposed GP-Generic. In Section~\ref{sec: concrete bounds conditional}, 
we present several concrete algorithms realized by Generic-GP. 

\textbf{Notation.} Let $\Ber(p)$ denote  a Bernoulli distribution  with success probability $p$, $\mathcal{N}(\mu, \sigma^2)$ denote a Gaussian distribution with mean $\mu$ and variance $\sigma^2$, and $\SG(\sigma^2)$ some zero-mean sub-Gaussian distribution with variance proxy $\sigma^2$.

\subsection{Exploration distributions} \label{sec: exploration distributions}

We say a sequence of data-dependent distributions $P_{w,1}, P_{w,2}, \dotsc, P_{w,t}$ can handle the exploitation-exploration trade-off efficiently, i.e., realize an algorithm enjoying  a sub-linear regret bound, if  this sequence of distributions satisfies all of the following conditions: 
\begin{gather}
            \mathbb{P}_{w_s \sim P_{w,s}} \left\{w_s \ge 1 \right\} = C_{1,s} \ne 0,\quad \forall s = 1,2, \dotsc, t
        \label{condition: exploration}
        \\
               \E \left[\mathop{\max}_{s \in [t]}|w_{s}| \right] = C_{2,t}
               \qquad
               \left(\E \left[\mathop{\max}_{s \in [t]     } \frac{|w_s|}{C_{1,s}} \right] + \mathop{\max}\limits_{s \in [t]     } \E \left[\frac{|w_s|}{C_{1,s}} \right]\right) = C_{3,t}\quad,
        \label{condition: exploitation}
\end{gather}
 where $w_1, w_2, \dotsc, w_t$ are independent random variables. 
 
 \paragraph{Remarks.}
 The choice of $P_{w,t}$  can depend on past information $\mathcal{F}_{t-1}$, but given this selection of distributions, all $w_1, w_2, \dotsc, w_t$ need to be independent. The condition in \eqref{condition: exploration} is to achieve optimism (exploring the unknown environment). Parameter $C_{2,t}$ is related to the regret caused by over-exploration of the sub-optimal arms, whereas parameter $C_{3,t}$ is related to the regret caused by under-exploration of the optimal arms. 
  Both $C_{2,t}, C_{3,t}$ can depend on $t$ and other problem-dependent parameters, such as the data dimension $d$ and the kernel parameters.
 Notice that we necessarily have
$C_{2,t} \le C_{2,t+1}$ and $C_{3,t} \le C_{3,t+1}$ for all $t \in [T]$.



\subsection{Generic-GP algorithm} \label{sec: Generic algorithm}
\begin{algorithm}[h]
\caption{Generic-GP Algorithm}\label{alg:GP-Gerneral}
\begin{algorithmic}[1]
\STATE Initialize: $V_0 = 0$
\FOR {round $t = 1,2, \dotsc$}
\STATE Choose exploration distribution $P_{w,t}$, which can depend on history information
\STATE Sample $w_t \sim P_{w,t}$; Compute $\tilde{f}_t$ based on (\ref{eq: randomized estimator})
\STATE Play $X_t = \mathop{\arg\max}_{x \in \mathcal{X}}\tilde{f}_{t}(x)$; Observe $Y_t = f^*(X_t) + \varepsilon_t$; Update $V_t$ and $\hat{f}_t$ based on (\ref{eq: empirical}).
\ENDFOR
\end{algorithmic}
\end{algorithm}

Recall that $V_t = \sum_{s = 1}^{t} K(\cdot, X_{s}) K(\cdot, X_{s})^{\top}   $  and 
$\overline{V}_t =  I_{H} + V_{t} $, where $I_H$ is the identity operator on $H_K$. Our proposed  Generic-GP algorithm is presented in Algorithm~\ref{alg:GP-Gerneral}. 
For exploitation, we construct an empirical function $  \hat{f}_{t-1} $ based on the past information, which can be expressed as
\begin{align}
  \hat{f}_{t-1}
  &= \overline{V}_{t-1}^{-1} \sum_{s = 1}^{t-1} Y_s K(\cdot, X_s)
\notag
\\&= \overline{V}_{t-1}^{-1} \sum_{s = 1}^{t-1}K(\cdot, X_s) \, (f^*(X_s) + \varepsilon_s)
\notag
\\&= \overline{V}_{t-1}^{-1} \sum_{s = 1}^{t-1}  K(\cdot, X_s)  K(\cdot, X_s)^{\top}f^* + \overline{V}_{t-1}^{-1} \sum_{s = 1}^{t-1} \varepsilon_s K(\cdot, X_s)
\quad.
    \label{eq: empirical}
\end{align}
To handle the exploitation-exploration trade-off, we add a random exploration function to  $\hat{f}_{t-1}$: 
\begin{equation}
\tilde{f}_t = \hat{f}_{t-1}  + w_{t} \, g_{t-1} \quad,
\label{eq: randomized estimator}
\end{equation}
where $w_t \sim P_{w,t}$ is an independent (scalar) random variable distributed according to the chosen exploration distribution in round $t$, and the function $g_{t-1}(x)$ 
summarizes the uncertainty associated with $\hat f_{t-1}$ as
\begin{equation}
\label{eq: g3}
           \quad g_{t-1}(x) = \left(\sqrt{2R^2 \log (2 \sqrt{ \det (\overline{V}_{t-1})}  )} + D \right)\cdot
           \left\lVert \overline{V}_{t-1}^{-1/2} K(\cdot, x) \right\rVert
       \quad.
\end{equation}
\textbf{Remarks.} Both functions $  \hat{f}_{t-1}$ and $g_{t-1}$ are determined by $\mathcal{F}_{t-1}$.  Using $\Ber(0)$  as exploration distribution gives $w_t \equiv 0$, hence $\tilde{f}_t = \hat{f}_{t-1}$ (pure exploitation, no exploration).  Using $\Ber(1)$ as exploration distribution gives $w_t \equiv 1$, hence  $\tilde{f}_t = \hat{f}_{t-1} + g_{t-1} $ (deterministic exploration).  
 Later sections  will show that using Bernoulli distributions with different parameters  enables us to use mixtures of distributions as exploration distributions, and facilitates the design of flexible algorithms. 

Now, we provide intuition behind the design of the weighting $2R^2 \log (2 \sqrt{ \det (\overline{V}_{t-1})}  )$. As already discussed in Sections~\ref{sec: intro} and \ref{sec: literature}, all $Y_1,  Y_2, \dotsc, Y_{t-1}$ are dependent random variables. Therefore, we cannot apply sub-Gaussian concentration bounds on $\overline{V}_{t-1}^{-1}\sum_{s=1}^{t-1} Y_s K(\cdot, X_s)$ directly.
As in previous work \citep{chowdhury2017kernelized,whitehouse2023sublinear,abbasi2011improved,agrawal2013thompson}, we  apply self-normalized bounds for vector-valued martingales \citep{whitehouse2023sublinear,abbasi2011improved}, stated in Theorem~\ref{concentration} on $\overline{V}_{t-1}^{-1} \sum_{s = 1}^{t-1} \varepsilon_s K(\cdot, X_s)$, but with a constant failure probability $\delta = 0.5$. The reason why we allow a constant failure probability $\delta = 0.5$ instead of a rate depending on $t$, is that the random variable  $w_t$ can help explore. This is quite different from UCB analyses.
\begin{theorem}\label{concentration}(Self-normalized concentration bounds for vector-valued martingales \citealp[Theorem 1]{whitehouse2023sublinear}.)
Consider the problem setup of \eqref{noise conditional}.
For any $\delta >0$, with probability at least $1-\delta$, for any $t \ge 0$, we have $ \left\lVert  \overline{V}_t^{-1/2} \sum_{s = 1}^{t}  \varepsilon_sK(\cdot, X_s) \right\rVert^2 \le 2R^2 \log\left( \sqrt{\det\left( \overline{V}_t\right)} /\delta \right)$.
\end{theorem}

\subsection{Regret bounds for Generic-GP algorithm} \label{sec: analysis}

In this section, we present generic regret bounds for Generic-GP (Algorithm~\ref{alg:GP-Gerneral}). 
We will analyze particular instances later.


\fakelabel{theorem}{theorem: expected regret}
\begin{restatable}{theorem}{thmexpregret}
    The  regret  of Algorithm~\ref{alg:GP-Gerneral} is 
    \begin{equation}
        \begin{array}{l}
       O \left( (C_{2,T} + C_{3,T})\sqrt{T\gamma_T} \left( \sqrt{R^2 \gamma_T  }  + D  \right)
             +  (\sqrt{R^2 \gamma_T + R^2 \ln(TD)}  + D  )\sqrt{T \gamma_T}  \right)\quad.      
        \end{array}
    \end{equation}
\end{restatable}
The kernel complexity $\gamma_T$ typically grows with $T$; for example, the Matérn kernel has $\gamma_T = O(T^{\frac{d}{2\nu +d}} \log(T))$.
Even so, choosing exploration distributions only based on  kernel parameters (e.g., smoothness parameter $\nu$) and the data dimension $d$ still enables the realization of sub-linear regret algorithms. For example, with the Matérn kernel with smoothness parameter $\nu$ and data dimension $d$, using exploration distributions satisfying
$C_{2, T} + C_{3,T} < o (T^{0.5 -\frac{d}{2 \nu +d}} )$
achieves an algorithm with sub-linear regret, i.e., an algorithm enjoying an $o(T) $ regret bounds.

\begin{proof}[Proof sketch of Theorem \ref{theorem: expected regret}; see Appendix \ref{app: main theorem} for details]   Let failure probability $\delta = O(1/TD^2)$. Define  $E_{t-1}$ as the event that $ \left\lVert  \overline{V}_{t-1}^{-1/2}\sum_{s = 1}^{t-1}  \varepsilon_s   K(\cdot, X_s)   \right\rVert \le  \sqrt{2R^2 \log ( \sqrt{\det\left( \overline{V}_{t-1}\right)} /\delta )} $, and  let $\overline{E}_{t-1} $ be its complement. 
We decompose the regret $\mathcal{R}(T)$ defined in \eqref{def: regret} as \[ \sum_{t=1}^{T}    \E \left[\left( f^*(x_*) - f^*(X_t) \right) \bm{1} \left\{E_{t-1} \right\} \right] + \sum_{t=1}^{T}    \E \left[\left( f^*(x_*) - f^*(X_t) \right) \bm{1} \left\{\overline{E}_{t-1} \right\} \right] ,\] where the second term is $O(1)$ by Theorem~\ref{concentration}, the self-normalized concentration bound for vector-valued martingales. Now, we only need to deal with
\begin{multline*} \label{dec}
         \sum_{t=1}^{T}    \E \left[\left( f^*(x_*) - f^*(X_t) \right) \bm{1} \left\{E_{t-1} \right\} \right] \\
         \le  \sum_{t=1}^{T}   \E [ \hat{f}_{t-1}(X_t) - f^*(X_t)  \bm{1} \left\{E_{t-1} \right\} ] + \E[\tilde{f}_{t}(X_t) -  \hat{f}_{t-1}(X_t)  ] + f^*(x_*)  - \E [ \tilde{f}_t(X_t) ],
\end{multline*}
where the first term is the cumulative estimation error, upper bounded by $O( \sqrt{R^2 \gamma_T + R^2 \ln(TD)}  + D  )\sqrt{T \gamma_T}$, and the second term $\E[\tilde{f}_{t}(X_t) -  \hat{f}_{t-1}(X_t)  ] \le \E\left[|w_t| \cdot g_{t-1}(X_t)  \right] $ is the approximation error, upper bounded by $C_{2,T} \cdot O( (\sqrt{R^2 \gamma_T }  + D)\sqrt{ T\gamma_T})$.
The last term above is challenging. Inspired by \citet{russo2019worst}, the value of $f^*(x_*)$ can be approximated as \[ \E [ \tilde{f}_{t}(X_t)]  +     2\E\left[ \frac{|w_t|}{C_{1,t}} \cdot  g_{t-1}(X_t)\right] + 2\E\left[\frac{|w_t|}{C_{1,t}}\right] \cdot \E [   g_{t-1}(X_t) ] ,\] which gives 
\[ \sum_{t=1}^{T} f^*(x_*) - \E [\tilde{f}_t(X_t)] 
               \le  C_{3,T} \cdot O (( \sqrt{R^2 \gamma_T }  + D )  \sqrt{ T\gamma_T}) .\qedhere\]\end{proof}







The regret analysis in Theorem~\ref{theorem: expected regret} is tight up to $\log$ factors, when viewing $C_{2,T}$ and $C_{3,T}$ as constants.  Section~\ref{sec: concrete bounds conditional} will show that, for some distributions,  the upper bounds for $C_{2,T}$ and $C_{3,T}$ are indeed constant. 
To show the tightness, we consider the linear bandit optimization problems studied by \citet{abbasi2011improved},
which cane be viewed as an istance of our problem setting with
a linear kernel $K(x, x') = x^{\top}x$. From \eqref{def: RKHS complexity} and Lemma~10 of \citet{abbasi2011improved}, the linear kernel has $\gamma_T = O ( d \log(T) )$.
We present theoretical results for Algorithm~\ref{alg:GP-Gerneral} using a linear kernel. 
\fakelabel{theorem}{theorem: linar bandit bound}
\begin{restatable}{theorem}{thmlinearbandit}
For the linear bandit optimization problems studied by \citet{abbasi2011improved}, the regret upper bound of Algorithm~\ref{alg:GP-Gerneral} is $O ( (C_{2,T} + C_{3,T})\sqrt{Td \log(T) } ( \sqrt{R^2 d \log(T)  }  + D  ))$.
    There exists a linear bandit optimization problem instance such that the regret is lower bounded by $\Omega\bigl(\sqrt{d T} (\sqrt d + D) \bigr)$.
\end{restatable}
Since there exists a problem instance with  nearly-matching upper and lower bounds, our analysis of Theorem~\ref{theorem: expected regret} is tight up to some extra $\log$ factors.


Now, we show that algorithms with sub-Gaussian exploration distributions  enjoy  $\tilde{O} (\gamma_T\sqrt{T} )$ regret bounds.
We will use the following well-known fact.
\begin{fact} \label{fact: maximal}
  (Maximal inequality). Let $X_1, X_2, \dotsc, X_n$ be $n$ random variables such that each $X_i \sim \SG(\sigma_i^2)$. Then $\E \left[\mathop{\max}_{i \in [n]} |X_i| \right] \le \sqrt{2 \log(2n)} \cdot \mathop{\max}_{i \in [n]} \sigma_i $.
  \end{fact}

 We present regret bounds for algorithms using sub-Gaussian distributions as exploration distributions.
\fakelabel{theorem}{coro: SG}
\begin{restatable}{theorem}{thmsubgauss}
    For an algorithm using a sequence of     sub-Gaussian distributions with parameters $\sigma_1^2, \sigma_2^2, \dotsc, \sigma_T^2$  and satisfying  the condition \eqref{condition: exploration} as exploration distributions, the regret bound of this specific algorithm is $O \left(\frac{\sqrt{T \log(T)\gamma_T}}{\min_{t \in [T]}C_{1,t}} \cdot \max_{t \in [T]} \sigma_t   \cdot \left( \sqrt{R^2 \gamma_T }  + D  \right)
                 \right)$.
 \end{restatable}
 Since sampling $w_t \sim \Ber(1)$   corresponds to  deterministic exploration,  we can design exploration distributions as mixtures of $\Ber(1)$ with other  distributions. Mixtures  lift the condition shown in (\ref{condition: exploration}).
 \fakelabel{theorem}{coro: mixture}
 \begin{restatable}{theorem}{thmmixture}
       For an algorithm such that in each round $t$,  we draw $w_t \sim \Ber(1)$ with probability $\rho_t > 0$ and  draw $w_t \sim \SG(\sigma_t^2)$ with probability $1-\rho_t$, the regret bound of this specific algorithm is
\begin{equation*}
    \begin{array}{l}
         O \left( \frac{\sqrt{T\gamma_T}}{\mathop{\min}_{t \in [T]} \rho_t}  \left(1+  \sqrt{ \log(T)}\mathop{\max}_{t \in [T]} \sigma_t  \right) \left( \sqrt{R^2 \gamma_T  }  + D  \right)
              \right). 
    \end{array}
\end{equation*}
\end{restatable}
Carefully choosing exploration distributions will realize algorithms enjoying the $\tilde{O} (\gamma_T\sqrt{T})$
regret bounds, which are optimal 
up to an extra $\sqrt{\gamma_T}$ factor as compared to the established  $\Omega(\sqrt{T\gamma_T})$ regret lower bounds \citep{scarlett2017lower,cai2025lower}. 

\subsection{Concrete GP algorithms} \label{sec: concrete bounds conditional}
In this section, we present several  concrete algorithms, \textbf{Simple-UCB}, \textbf{Simple-Bernoulli}, \textbf{Simple-Gaussian}, and \textbf{Simple-Categorical},  based on simple  exploration distributions. These exploration distributions all realize computationally efficient algorithms that achieve $\tilde{O}(\gamma_T\sqrt{T} )$ regret bounds.


\paragraph{Simple-UCB.} \label{sec: Simple-UCB}

Using $\Ber(1)$  as exploration distribution in all rounds, i.e., $w_t \equiv 1$ for all $t \in [T]$, realizes a deterministic UCB-based algorithm, \textbf{Simple-UCB}. 







\begin{theorem}
    The regret of  \textbf{Simple-UCB} is $O \left(\sqrt{T\gamma_T} \left( \sqrt{R^2 \gamma_T  }  + D  
             + \sqrt{ R^2 \ln(TD)}   \right)\right)$.
\end{theorem}
\begin{proof}
For $\Ber(1)$, we have $C_{1,t} = 1$. Then, we have $C_{2,T} \le 1$ and $C_{3,T} \le 2 $. Plugging the upper bounds for $C_{2,T}$ and $C_{3,T}$ into Theorem~\ref{theorem: expected regret} concludes the proof.\end{proof}
From Jensen's inequality, we have $C_{3,T} =  \E [\mathop{\max}_{t \in [T]     } \frac{|w_t|}{C_{1,t}} ] + \mathop{\max}_{t \in [T]     } \E [\frac{|w_t|}{C_{1,t}} ] \ge  \mathop{\max}_{t \in [T]     } \E [ \frac{|w_t|}{C_{1,t}} ] + \mathop{\max}_{t \in [T]     } \E [\frac{|w_t|}{C_{1,t}} ] \ge 2$. Combining with the fact that $C_{3,T} \le 2$, we know that using $\Ber(1)$ as exploration distributions in all rounds is the best choice to control $C_{3,T}$, but  $\Ber(1)$ may not be the best choice to control $C_{2,T} =  \E \left[\mathop{\max}_{t \in [T]}|w_{t}|  \right] $. 
Our \textbf{Simple-UCB} achieves the same regret bound as  GP-UCB of \citet{whitehouse2023sublinear} and IGP-UCB of \citet{chowdhury2017kernelized}. 
\textbf{Simple-UCB} is more computationally efficient, however, as it does not need to construct the confidence balls used by these algorithms.


\paragraph{Simple-Bernoulli.}               Using $\Ber(p_1), \Ber(p_2), \dotsc, \Ber(p_T) $ with each $p_t \in (0, 1]$ as the sequence of exploration distributions realizes \textbf{Simple-Bernoulli}, where we have $\tilde{f}_t=\hat{f}_{t-1}+ g_{t-1}$ with probability $p_t$   and   $\tilde{f}_t=\hat{f}_{t-1}$ with probability $1-p_t$ in each round $t$.
\begin{theorem}
    The regret of \textbf{Simple-Bernoulli} is $O \left( \frac{\sqrt{T\gamma_T} }{\mathop{\min}_{t \in [T]} p_t} \left( \sqrt{R^2 \gamma_T  }  + D  
             + \sqrt{ R^2 \ln(TD)}   \right)
               \right)$.
\end{theorem}
\begin{proof}
    For  $\Ber(p_t)$, we have $C_{1,t} = p_t$. Then, 
     we have $C_{2,T} =  \E [\mathop{\max}_{t \in [T]} |w_t| ] \le 1 $ and $C_{3,T} = \E [\mathop{\max}_{t \in [T]     } \frac{|w_t|}{C_{1,t}} ] + \mathop{\max}_{t \in [T]     } \E [\frac{|w_t|}{C_{1,t}} ] \le \frac{1}{\mathop{\min}_{t \in [T]} p_t} +1 \le \frac{2}{\mathop{\min}_{t \in [T]} p_t}$. Plugging the upper bounds for $C_{2,T}$ and $C_{3,T}$ into Theorem~\ref{theorem: expected regret} concludes the proof.\end{proof}
  \textbf{Simple-Bernoulli} is an extremely  efficient and practical algorithm, as we can adjust  $p_t$ to control the exploration level based on the seen information. Once we are confident about the learning environment, we can choose a smaller $p_t$ to force the algorithm to focus on  exploitation. 

\paragraph{Simple-Gaussian.} \label{sec: Simple-Gaussian}

Using $ w_t \sim  \mathcal{N}(0,1)$  as the exploration distribution
realizes \textbf{Simple-Gaussian}.
Contrary to the previous algorithms,
here we may have $w_t < 0$,
and thus sometimes choose points with the highest \emph{lower} confidence bound.
\begin{theorem}
    The regret of \textbf{Simple-Gaussian} is  $ O \left(\sqrt{T \log(T)\gamma_T}    \left(\sqrt{R^2 \gamma_T } + D  \right) 
            \right)$.
\end{theorem}
\begin{proof}From Fact~\ref{fact: Gaussian bounds}, we have $C_{1,t} = \mathbb{P}\left\{w \ge 1 \right\} \ge \frac{1}{\sqrt{2\pi}} \frac{1}{2} e^{-0.5}$. Plugging the lower bound for $C_{1,t}$ into Theorem~\ref{coro: SG} concludes the proof. \end{proof}
Compared to the $\tilde{O} (\gamma_T \sqrt{Td})$ regret bound of GP-TS \citep{chowdhury2017kernelized}, our \textbf{Simple-Gaussian} saves on an extra $\sqrt{d}$ factor.
Rather than sampling from a $d$-dimensional Gaussian,
we explore based on a one-dimensional Gaussian,
similar to the approach of \citet{xiong2022near}.

\paragraph{Simple-Categorical.} \label{sec: Generic Categorical}
Consider discrete distributions supported on the set
$\{\frac{1}{K_t}, \dotsc, \frac{i}{K_t}, \dotsc, \frac{K_t}{K_t}\}$
with parameter
$\Vec{p}_t = (p_t^{(1)}, \dotsc, p_t^{(i)}, \dotsc, p_t^{(K_t)} )$,
i.e.\ 
$\mathbb{P}_{w_t }\left\{w_t = \frac{i}{K_t} \right\} = p_t^{(i)}$ for all $i \in [K_t]$,
with each probability positive and summing to 1.
  Using a sequence of
  these distributions
  realizes \textbf{Simple-Categorical}, where we have $\tilde{f}_t  = \hat{f}_{t-1} +  \frac{i}{K_t} \cdot g_{t-1}$ with probability $p_t^{(i)}$ for all $i \in [K_t]$.
\begin{theorem}
    The regret of \textbf{Simple-Categorical} is $   O \left(\frac{\sqrt{T\gamma_T} }{\mathop{\min}_{t \in [T]} p_t^{(K_t)}}\left( \sqrt{R^2 \gamma_T  }  + D  
             + \sqrt{ R^2 \ln(TD)}   \right)
            \right)$.
\end{theorem}
\begin{proof}
We have $C_{1,t} = p_t^{(K_t)}$. Then, 
we have $C_{2,T} =  \E [\mathop{\max}_{t \in [T]} |w_t| ] \le 1 $ and $C_{3,T} = \E [\mathop{\max}_{t \in [T]     } \frac{|w_t|}{C_{1,t}} ] + \mathop{\max}_{t \in [T]     } \E [\frac{|w_t|}{C_{1,t}} ] \le \frac{2}{\mathop{\min}_{t \in [T]} p_t^{(K_t)}}$. Plugging the upper bounds for $C_{2,T}$ and $C_{3,T}$ into Theorem~\ref{theorem: expected regret} concludes the proof.\end{proof}












\section{Experimental Results} \label{sec: experiments}
We conduct experiments to compare the empirical performance of our proposed concrete algorithms with IGP-UCB and GP-TS \citep{chowdhury2017kernelized} by using both synthetic data and real-world data. For our proposed algorithms, we use the following exploration distributions: 
\textbf{Simple-UCB}, \textbf{Simple-Gaussian}, \textbf{Simple-Bernoulli} with the first $T/2$ rounds using $\Ber(0.5)$ and the second $T/2$ rounds using $\Ber(0.25)$,
and \textbf{Simple-Categorical} with support $\{\frac{1}{4}, \frac{2}{4}, \frac{3}{4}, \frac{4}{4} \}$ and parameter $(\frac{1}{4}, \frac{1}{4}, \frac{1}{4}, \frac{1}{4} )$.
We run each experiment for $T=10^3$ rounds and repeat $25$ independent random-seed runs.  We use a Gaussian RBF kernel with regularization parameter $\rho = 1$. 

We evaluate on four synthetic test functions: \textbf{Holder Table (dimension $d =2$), Cross-in-Tray (dimension $d =2$), Ackley (dimension $d =4$), Hartmann (dimension $d =6$)}, whose analytical forms are shown at \url{https://www.sfu.ca/\~ssurjano/optimization.html}.
Each task has a 50-point candidate set  and is corrupted by light Gaussian noise (\(\sigma = 10^{-2}\)).

We also use the real-world Perovskite dataset \citep{sun2021data}, which has 94 samples spanning 3 dimensions. 

Figure~\ref{fig:five-2row} reports the experimental results of all the compared algorithms using both synthetic and real-world data.
For experiments  over synthetic data (Figure~\ref{fig:plot1} to \ref{fig:plot4}),
our proposed Simple-Bernoulli is extremely efficient and achieves the best empirical performance. This is because it enables reducing exploration level as information is gathered. Our proposed Simple-Gaussian outperforms GP-TS empirically, although both of them use Gaussian distributions to do exploration. Simple-Categorical outperforms both IGP-UCB and Simple-UCB. For the real-world Perovskite dataset (Figure~\ref{fig:plot5}), our proposed Simple-Gaussian, Simple-Bernoulli, and Simple-Categorical perform similarly, and outperforms IGP-UCB, Simple-UCB, and GP-TS.

\begin{figure}[t]
  \centering
  \subfigure[Holder Table (d = 2)]{
    \includegraphics[width=0.33\linewidth]{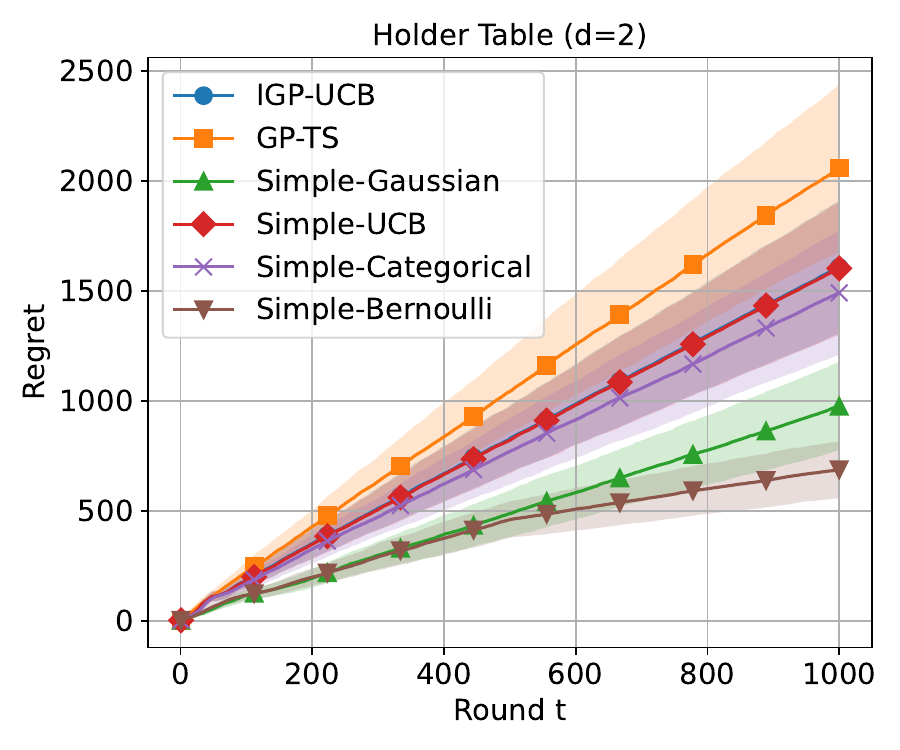}
    \label{fig:plot1}
  }\hspace{-1cm}
  \hfill
  \subfigure[Cross-in-Tray (d  =2)]{
    \includegraphics[width=0.33\linewidth]{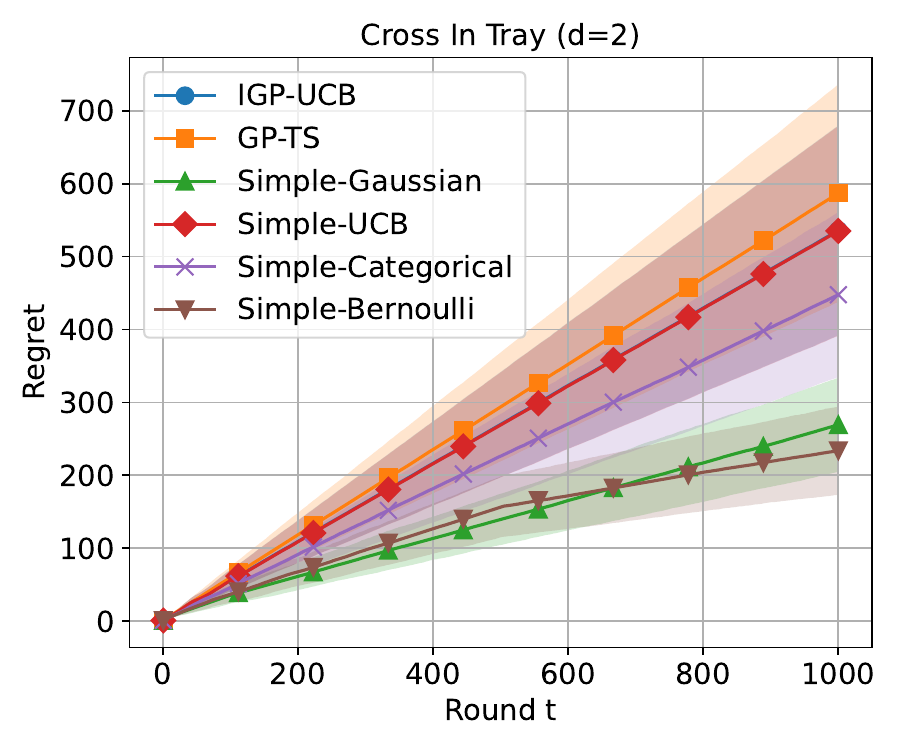}
    \label{fig:plot2}
  }\hspace{-1cm}
  \hfill
  \subfigure[Ackley (d = 4)]{
    \includegraphics[width=0.33\linewidth]{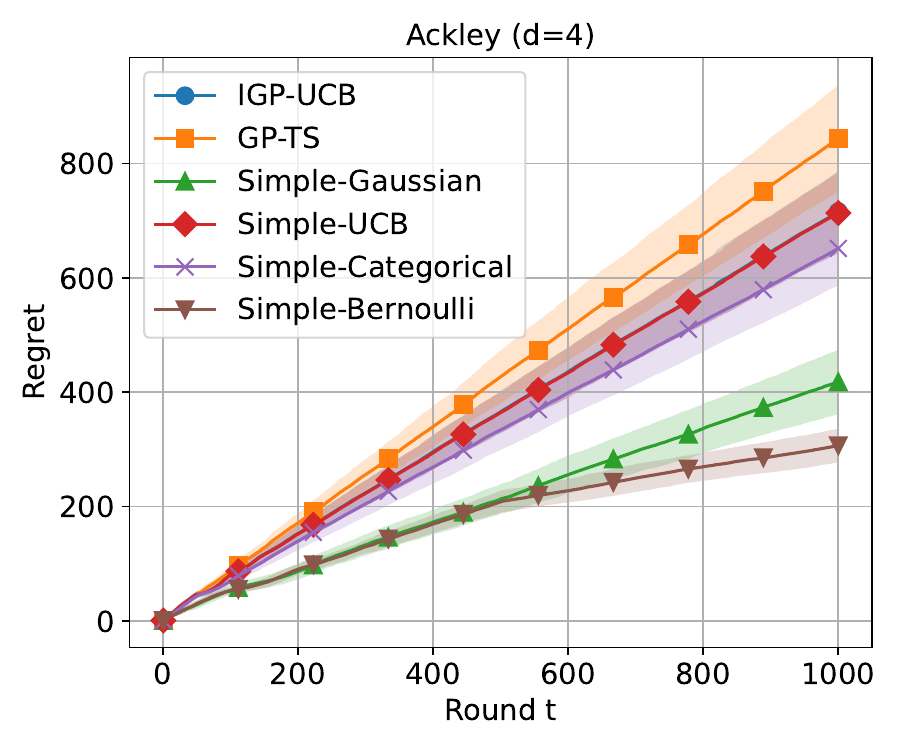}
    \label{fig:plot3}
  }
  \subfigure[Hartmann (d = 6)]{
    \includegraphics[width=0.33\linewidth]{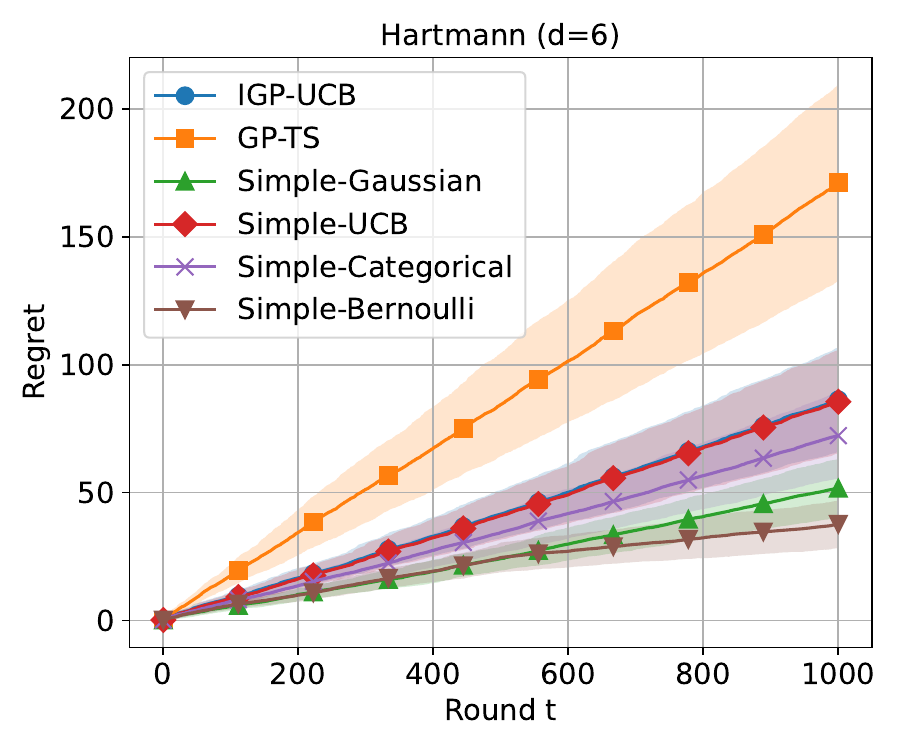}
    \label{fig:plot4}
  }
  \subfigure[Perovskite real-world dataset]{
    \includegraphics[width=0.33\linewidth]{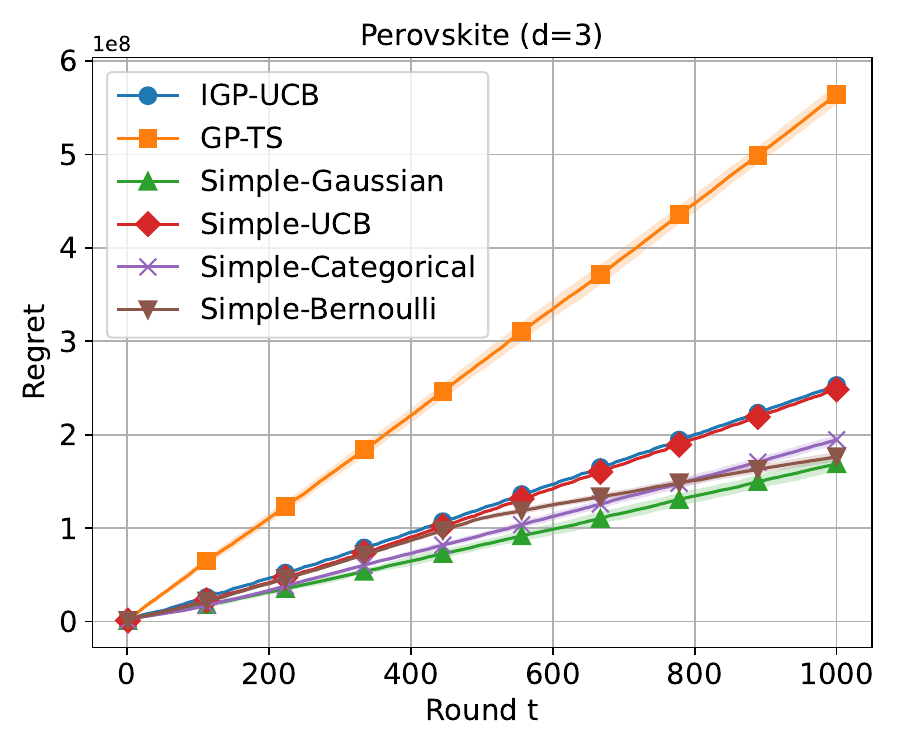}
    \label{fig:plot5}
  }

  \caption{Regret comparison of different algorithms for synthetic and real-world data}
  \label{fig:five-2row}
\end{figure}

 \section{Conclusion and Future work}
Our novel kernelized bandit algorithms add flexibility, achieve the same regret bounds as others in their class with a unified analysis, and perform better in practice.
Simple-Bernoulli in particular seems an extremely promising and exceedingly simple
improvement over current UCB-type algorithms.
 
As discussed by \citet{lattimore2023lower}, we think that optimal $O(\sqrt{T \gamma_T})$ regret bounds may not be obtainable if in each round $t$, the pulled arm $X_t$ is distributed according to a data-dependent distribution. 
To develop algorithms achieving better regret bounds, we must  carefully control the arm sampling path. On one hand, we would like the generated reward sequence $Y_1, \dotsc, Y_t$ to be dependent as this phenomenon implies  the developed algorithms utilize the past information well. On the other hand, we would like the reward sequence $Y_1, \dotsc, Y_t$ to be less dependent, which will help us to use concentration bounds in a low-dimensional vector-space. Regarding future research, we plan to investigate the possibility to develop a generic algorithm with high probability regret guarantees. Time-varying kernelized bandits as studied by \citet{cai2025lower} are also an interesting extension. 










\begin{ack}
This work was funded in part by
the Natural Sciences and Engineering Resource Council of Canada,
the Canada CIFAR AI chairs program,
and the UBC Data Science Institute.
\end{ack}

\printbibliography

@article{seeger2004gaussian,
  title={Gaussian processes for machine learning},
  author={Seeger, Matthias},
  journal={International journal of neural systems},
  volume={14},
  number={02},
  pages={69--106},
  year={2004},
  publisher={World Scientific}
}

@article{russo2019worst,
  title={Worst-case regret bounds for exploration via randomized value functions},
  author={Russo, Daniel},
  journal={Advances in Neural Information Processing Systems},
  year={2019}
}

@article{sun2021data,
  title={A data fusion approach to optimize compositional stability of halide perovskites},
  author={Sun, Shijing and Tiihonen, Armi and Oviedo, Felipe and Liu, Zhe and Thapa, Janak and Zhao, Yicheng and Hartono, Noor Titan P and Goyal, Anuj and Heumueller, Thomas and Batali, Clio and others},
  journal={Matter},
  volume={4},
  number={4},
  pages={1305--1322},
  year={2021},
  publisher={Elsevier}
}

@inproceedings{vakili2022open,
  title={{O}pen problem: {R}egret bounds for noise-free kernel-based bandits},
  author={Vakili, Sattar},
  booktitle={Conference on Learning Theory},
  pages={5624--5629},
  year={2022},
  organization={PMLR}
}

@article{vakili2021optimal,
  title={Optimal order simple regret for {G}aussian process bandits},
  author={Vakili, Sattar and Bouziani, Nacime and Jalali, Sepehr and Bernacchia, Alberto and Shiu, Da-shan},
  journal={Advances in Neural Information Processing Systems},
  volume={34},
  pages={21202--21215},
  year={2021}
}

@inproceedings{li2022gaussian,
  title={Gaussian process bandit optimization with few batches},
  author={Li, Zihan and Scarlett, Jonathan},
  booktitle={International Conference on Artificial Intelligence and Statistics},
  pages={92--107},
  year={2022},
  organization={PMLR}
}

@inproceedings{vakili2021open,
  title={{O}pen problem: {T}ight online confidence intervals for RKHS elements},
  author={Vakili, Sattar and Scarlett, Jonathan and Javidi, Tara},
  booktitle={Conference on Learning Theory},
  pages={4647--4652},
  year={2021},
  organization={PMLR}
}

@inproceedings{srinivas2010gaussian,
  title={Gaussian process optimization in the bandit setting: no regret and experimental design},
  author={Srinivas, Niranjan and Krause, Andreas and Kakade, Sham and Seeger, Matthias},
  booktitle={International Conference on Machine Learning},
  pages={1015--1022},
  year={2010}
}

@article{auer2002using,
  title={Using confidence bounds for exploitation-exploration trade-offs},
  author={Auer, Peter},
  journal={Journal of Machine Learning Research},
  volume={3},
  number={Nov},
  pages={397--422},
  year={2002}
}

@article{ray2019bayesian,
  title={Bayesian optimization under heavy-tailed payoffs},
  author={Ray Chowdhury, Sayak and Gopalan, Aditya},
  journal={Advances in Neural Information Processing Systems},
  volume={32},
  year={2019}
}

@inproceedings{valko2013finite,
  title={Finite-time analysis of kernelised contextual bandits},
  author={Valko, Michal and Korda, Nathan and Munos, R{\'e}mi and Flaounas, Ilias and Cristianini, Nello},
  booktitle={Uncertainty in Artificial Intelligence},
  year={2013}
}

@inproceedings{shekhar2022instance,
  title={Instance dependent regret analysis of kernelized bandits},
  author={Shekhar, Shubhanshu and Javidi, Tara},
  booktitle={International Conference on Machine Learning},
  pages={19747--19772},
  year={2022},
  organization={PMLR}
}

@inproceedings{lee2022multi,
  title={Multi-scale zero-order optimization of smooth functions in an {RKHS}},
  author={Lee, Madison and Shekhar, Shubhanshu and Javidi, Tara},
  booktitle={2022 IEEE International Symposium on Information Theory (ISIT)},
  pages={288--293},
  year={2022},
  organization={IEEE}
}

@article{shekhar2018gaussian,
  title={Gaussian process bandits with adaptive discretization},
  author={Shekhar, S and Javidi, T},
  journal={Electronic journal of statistics},
  volume={12},
  number={2},
  year={2018}
}

@inproceedings{cai2025lower,
title={Lower bounds for time-varying kernelized bandits},
author={Xu Cai and Jonathan Scarlett},
booktitle={International Conference on Artificial Intelligence and Statistics},
year={2025},
url={}
}

@inproceedings{lattimore2023lower,
  title={A lower bound for linear and kernel regression with adaptive covariates},
  author={Lattimore, Tor},
  booktitle={Conference on Learning Theory},
  pages={2095--2113},
  year={2023},
  organization={PMLR}
}

@article{durand2018streaming,
  title={Streaming kernel regression with provably adaptive mean, variance, and regularization},
  author={Durand, Audrey and Maillard, Odalric-Ambrym and Pineau, Joelle},
  journal={Journal of Machine Learning Research},
  volume={19},
  number={17},
  pages={1--34},
  year={2018}
}

@inproceedings{vakili2021information,
  title={On information gain and regret bounds in gaussian process bandits},
  author={Vakili, Sattar and Khezeli, Kia and Picheny, Victor},
  booktitle={International Conference on Artificial Intelligence and Statistics},
  pages={82--90},
  year={2021},
  organization={PMLR}
}

@inproceedings{scarlett2017lower,
  title={Lower bounds on regret for noisy gaussian process bandit optimization},
  author={Scarlett, Jonathan and Bogunovic, Ilija and Cevher, Volkan},
  booktitle={Conference on Learning Theory},
  pages={1723--1742},
  year={2017},
  organization={PMLR}
}

@article{flynn2024tighter,
  title={Tighter confidence bounds for sequential kernel regression},
  author={Flynn, Hamish and Reeb, David},
  journal={arXiv preprint arXiv:2403.12732},
  year={2024}
}

@inproceedings{vakili2024open,
  title={{O}pen problem: {O}rder optimal regret bounds for kernel-based reinforcement learning},
  author={Vakili, Sattar},
  booktitle={Conference on Learning Theory},
  pages={5340--5344},
  year={2024},
  organization={PMLR}
}

@inproceedings{salgia2024random,
  title={Random exploration in Bayesian optimization: {O}rder-optimal regret and computational efficiency},
  author={Salgia, Sudeep and Vakili, Sattar and Zhao, Qing},
  booktitle={International Conference on Machine Learning},
  pages={43112--43141},
  year={2024},
  organization={PMLR}
}

@article{kim2024enhancing,
  title={Enhancing Gaussian process surrogates for optimization and posterior approximation via random exploration},
  author={Kim, Hwanwoo and Sanz-Alonso, Daniel},
  journal={SIAM/ASA Journal on Uncertainty Quantification},
  year={2025}
}

@article{xiong2022near,
  title={Near-optimal randomized exploration for tabular markov decision processes},
  author={Xiong, Zhihan and Shen, Ruoqi and Cui, Qiwen and Fazel, Maryam and Du, Simon S},
  journal={Advances in neural information processing systems},
  volume={35},
  pages={6358--6371},
  year={2022}
}

@inproceedings{chowdhury2017kernelized,
  title={On kernelized multi-armed bandits},
  author={Chowdhury, Sayak Ray and Gopalan, Aditya},
  booktitle={International Conference on Machine Learning},
  pages={844--853},
  year={2017},
  organization={PMLR}
}

@inproceedings{janz2020bandit,
  title={Bandit optimisation of functions in the Mat{\'e}rn kernel {RKHS}},
  author={Janz, David and Burt, David and Gonz{\'a}lez, Javier},
  booktitle={International Conference on Artificial Intelligence and Statistics},
  pages={2486--2495},
  year={2020},
  organization={PMLR}
}

@article{whitehouse2023sublinear,
  title={On the sublinear regret of {GP-UCB}},
  author={Whitehouse, Justin and Ramdas, Aaditya and Wu, Steven Z},
  journal={Advances in Neural Information Processing Systems},
  volume={36},
  pages={35266--35276},
  year={2023}
}

@article{abbasi2011improved,
  title={Improved algorithms for linear stochastic bandits},
  author={Abbasi-Yadkori, Yasin and P{\'a}l, D{\'a}vid and Szepesv{\'a}ri, Csaba},
  journal={Advances in neural information processing systems},
  volume={24},
  year={2011}
}

@article{auer2010ucb,
  title={UCB revisited: {I}mproved regret bounds for the stochastic multi-armed bandit problem},
  author={Auer, Peter and Ortner, Ronald},
  journal={Periodica Mathematica Hungarica},
  volume={61},
  number={1-2},
  pages={55--65},
  year={2010},
  publisher={Akad{\'e}miai Kiad{\'o}, co-published with Springer Science+ Business Media BV~…}
}

@inproceedings{chu2011contextual,
  title={Contextual bandits with linear payoff functions},
  author={Chu, Wei and Li, Lihong and Reyzin, Lev and Schapire, Robert},
  booktitle={Proceedings of the Fourteenth International Conference on Artificial Intelligence and Statistics},
  pages={208--214},
  year={2011},
  organization={JMLR Workshop and Conference Proceedings}
}

@inproceedings{agrawal2013thompson,
  title={{Thompson Sampling} for contextual bandits with linear payoffs},
  author={Agrawal, Shipra and Goyal, Navin},
  booktitle={International conference on machine learning},
  pages={127--135},
  year={2013},
  organization={PMLR}
}

@article{auer2002finite,
  title={Finite-time analysis of the multi-armed bandit problem},
  author={Auer, Peter and Cesa-Bianchi, Nicolo and Fischer, Paul},
  journal={Machine learning},
  volume={47},
  pages={235--256},
  year={2002},
  publisher={Springer}
}

@misc{agrawalnear,
  author = {Agrawal, Shipra and Goyal, Navin},
  title = {{N}ear-optimal regret bounds for {T}hompson {S}ampling},
  howpublished = "\url{http://www.columbia.edu/~sa3305/papers/j3-corrected.pdf}",
  year = {2017}, 
}

@inproceedings{bandit-apps,
author = {Bouneffouf, Djallel and Rish, Irina and Aggarwal, Charu},
title = {Survey on Applications of Multi-Armed and Contextual Bandits},
year = {2020},
booktitle = {2020 IEEE Congress on Evolutionary Computation (CEC)},
}

@book{garnett_bayesoptbook_2023,
  author    = {Garnett, Roman},
  title     = {{Bayesian Optimization}},
  year      = {2023},
  publisher = {Cambridge University Press}
}

\appendix

\newpage
The appendix is organized as follows.
\begin{enumerate}
    \item Appendix~\ref{app: facts} presents facts, concentration and anti-concentration bounds used in this work.
    \item Appendix~\ref{app: main theorem} presents proof of Theorem~\ref{theorem: expected regret}.
    \item Appendix~\ref{app: linar bandit bound} presents proof of Theorem~\ref{theorem: linar bandit bound}.
    \item Appendix~\ref{app: SG} presents proof of Theorem~\ref{coro: SG}.
    \item Appendix~\ref{app: mixture} presents proof of Theorem~\ref{coro: mixture}.
    
\end{enumerate}
  \section{Useful facts, concentration and anti-concentration bounds} \label{app: facts}

  \begin{fact} \label{fact: Gaussian bounds}
  (Gaussian concentration and anti-concentration bounds; well-known.) For a Gaussian distributed random variable $Z$ with mean $\mu$ and variance $\sigma^2$, for any $z > 0$, we have
\begin{equation}
\begin{array}{l}
 \mathbb{P} \left\{Z > \mu + z \sigma \right\} \le \frac{1}{2}e^{- \frac{z^2}{2}}, \quad \mathbb{P} \left\{Z < \mu - z \sigma \right\} \le \frac{1}{2}e^{- \frac{z^2}{2}}\quad,
 \end{array}
\end{equation}
and 
\begin{equation}
\begin{array}{l}
  \mathbb{P} \left\{Z > \mu + z \sigma \right\} \ge \frac{1}{\sqrt{2\pi}} \frac{z}{z^2+1} e^{- \frac{z^2}{2}} \quad.
  \end{array}
\end{equation}      
  \end{fact}



\begin{fact}
    (Lemma~5 in \citet{whitehouse2023sublinear}.) For any round $t \ge 1$, we have
    \begin{equation}
    \begin{array}{lll}
     \sum\limits_{s=1}^{t} \left\lVert \overline{V}_{s-1}^{-1/2} K(\cdot, X_s)\right\rVert^2 & \le & 2 \log \left( \det \left(I_H + V_t \right) \right) = 2 \log \left( \det \left(\overline{V}_t \right) \right) \quad.
         \end{array}
    \end{equation}
    \label{lemma: elliptical potential}
\end{fact}
\section{Proof of Theorem~\ref{theorem: expected regret}} \label{app: main theorem}
\thmexpregret*
\begin{proof}[Proof of Theorem~\ref{theorem: expected regret}]
   Let $\delta = \frac{1}{2TD^2}$. Define  $E_{t-1}$ as the event that \[ \left\lVert  \overline{V}_{t-1}^{-1/2}\sum\limits_{s = 1}^{t-1}  \varepsilon_s   K(\cdot, X_s)   \right\rVert \le  \sqrt{2R^2 \log \left( \sqrt{\det\left( \overline{V}_{t-1}\right)} /\delta \right)} \quad,\]  and let $\overline{E}_{t-1} $ be its complement.

We decompose the regret based on $E_{t-1}$ and $\overline{E}_{t-1} $. We have
 \begin{equation}
    \begin{array}{lll}
       \mathcal{R}(T)   & = &\sum\limits_{t=1}^{T}    \E \left[ f^*(x_*) - f^*(X_t) \right] \\
  &  = &\sum\limits_{t=1}^{T}    \E \left[\left( f^*(x_*) - f^*(X_t) \right) \bm{1} \left\{E_{t-1} \right\} \right] + \sum\limits_{t=1}^{T}    \E \left[\left( f^*(x_*) - f^*(X_t) \right) \bm{1} \left\{\overline{E}_{t-1} \right\} \right].
\end{array}
\label{temp 100}
  \end{equation}
The second term above is upper bounded by 
\begin{equation}
    \begin{array}{lll}
         \sum\limits_{t=1}^{T}    \E \left[\left( f^*(x_*) - f^*(X_t) \right) \bm{1} \left\{\overline{E}_{t-1} \right\} \right] 
     &  =  & \sum\limits_{t=1}^{T}    \E \left[\left\langle K(\cdot, x_*) - K(\cdot, X_t), f^*  \right\rangle \bm{1} \left\{\overline{E}_{t-1} \right\} \right] \\
& \le   & \sum\limits_{t=1}^{T}    \E \left[\left( \left\lVert K(\cdot, x_*) \right\rVert +\left\lVert K(\cdot, X_t) \right\rVert \right) \cdot  \left\lVert f^*  \right\rVert \bm{1} \left\{\overline{E}_{t-1} \right\} \right] \\    
& \le   & \sum\limits_{t=1}^{T}    \E \left[2D^2 \cdot \bm{1} \left\{\overline{E}_{t-1} \right\} \right] \\
&\le   & 2D^2 \sum\limits_{t=1}^{T}    \underbrace{\mathbb{P}  \left\{\overline{E}_{t-1} \right\}}_{\text{Theorem~\ref{concentration}}} \\
& \le & 2D^2T\delta \\
& \le & 1\quad,
    \end{array}
\end{equation}
where the second last inequality uses 
Theorem~\ref{concentration}, stating that, for any $t \ge 1$, we have $\mathbb{P}\left\{\overline{E}_{t-1}  \right\} \le \delta$.

  \

Now, we continue decomposing the regret in first term of (\ref{temp 100}) by introducing $\pm \hat{f}_{t-1}(X_t)$ and $\pm \tilde{f}_t(X_t)$. We have
\begin{multline}
    \sum\limits_{t=1}^{T}    \E \left[\left( f^*(x_*) - f^*(X_t) \right)  \cdot \bm{1} \left\{E_{t-1} \right\}\right]
      \le \underbrace{\sum_{t=1}^{T} \E \left[ \left(\hat{f}_{t-1}(X_t) - f^*(X_t) \right)  \bm{1} \left\{E_{t-1} \right\} \right]}_{I_3}
      \\
      + \underbrace{\sum_{t=1}^{T} \E\left[\tilde{f}_{t}(X_t) -  \hat{f}_{t-1}(X_t)  \right]}_{I_2}
      + 
 \underbrace{\sum_{t=1}^{T} \E \left[ f^*(x_*)  -  \tilde{f}_t(X_t)  \right] }_{I_1}.
    \end{multline}
We prepare a lemma for each term above.
\begin{lemma}
We have 
  \[
      I_3 = \sum\limits_{t=1}^{T}\E \left[\left(\hat{f}_{t-1}(X_t) - f^*(X_t)\right) \cdot  \bm{1} \left\{E_{t-1} \right\}   \right] \le \left( \sqrt{2R^2 \gamma_T + 2R^2 \ln(2TD^2)}  + D  \right)\sqrt{4T \gamma_T}.
\]
    \label{lemma: I_3}
    \end{lemma}

\begin{lemma}
    We have 
    \[
          I_2 =  \sum\limits_{t=1}^{T}\E \left[ \tilde{f}_{t}(X_t) -  \hat{f}_{t-1}(X_t) \right] \le C_{2,T}  \cdot \left( \sqrt{2R^2 \gamma_T + 2R^2\log(2) }  + D  \right)  \cdot  \sqrt{ 4T\gamma_T}.
    \]
    \label{lemma: I_2}
    \end{lemma}
\begin{lemma}
    We have 
    \[
             I_1  =  \sum\limits_{t=1}^{T} \E \left[f^*(x_*) - \tilde{f}_t(X_t)  \right] 
               \le  2C_{3,T} \cdot  \left( \sqrt{2R^2 \gamma_T + 2R^2\log(2) }  + D  \right)  \cdot  \sqrt{ 4T\gamma_T}.
    \] \label{lemma: I_1}
\end{lemma}

Combining the above results gives that $\mathcal R(T)$ is at most
\begin{multline*}
            2(C_{2,T} + 2C_{3,T}) \left( \sqrt{2R^2 \gamma_T + 2R^2\log(2) }  + D  \right)\sqrt{T\gamma_T}
             + 2\left( \sqrt{2R^2 \gamma_T + 2R^2 \ln(2TD^2)}  + D  \right)\sqrt{ T \gamma_T } + 1\\
              = O \left( (C_{2,T} + C_{3,T})\sqrt{T\gamma_T} \left( \sqrt{R^2 \gamma_T  }  + D  \right)
             +  (\sqrt{R^2 \gamma_T + R^2 \ln(TD)}  + D  )\sqrt{T \gamma_T}  \right),
    \end{multline*}
 which concludes the proof.\end{proof}

\subsection{Proof of Lemma~\ref{lemma: I_3}  }

\begin{proof}[Proof of Lemma~\ref{lemma: I_3}]

Before presenting the detailed proof, we first construct an upper bound  for $   \left\lvert \hat{f}_{t-1}(X_t)  - f^*(X_t) \right\rvert$.   From 
    \begin{align*}
    &  \hat{f}_{t-1}(X_t)  - f^*(X_t) \\
     &=  \left\langle \overline{V}_{t-1}^{-1} \sum\limits_{s = 1}^{t-1}  K(\cdot, X_s)  K(\cdot, X_s)^{\top}f^* + \overline{V}_{t-1}^{-1} \sum\limits_{s = 1}^{t-1} \varepsilon_s K(\cdot, X_s) -   f^*,  K(\cdot, X_t) \right\rangle  \\
    &=  \left\langle  \overline{V}_{t-1}^{-1}  \sum\limits_{s = 1}^{t-1}  \varepsilon_s   K(\cdot, X_s),  K(\cdot, X_t) \right\rangle \\
    &+  \left\langle \overline{V}_{t-1}^{-1} \sum\limits_{s = 1}^{t-1}  K(\cdot, X_s)  K(\cdot, X_s)^{\top}f^* -  \overline{V}_{t-1}^{-1}  f^*-  \overline{V}_{t-1}^{-1} \sum\limits_{s=1}^{t-1} K(\cdot, X_s)K(\cdot, X_s)^{\top} f^*,  K(\cdot, X_t) \right\rangle  \\
                    &=  \left\langle  \overline{V}_{t-1}^{-1}  \sum\limits_{s = 1}^{t-1}  \varepsilon_s   K(\cdot, X_s),  K(\cdot, X_t)  \right\rangle -  \left\langle    \overline{V}_{t-1}^{-1}  f^*,  K(\cdot, X_t) \right\rangle
            \quad,
          \end{align*}
        we have
          \begin{align}
                    \left\lvert \hat{f}_{t-1}(X_t)  - f^*(X_t) \right\rvert & \le \left\lVert K(\cdot, X_t)^{\top}\overline{V}_{t-1}^{-1/2} \right\rVert \cdot \left\lVert \overline{V}_{t-1}^{-1/2}  \sum\limits_{s = 1}^{t-1}  \varepsilon_s   K(\cdot, X_s) \right\rVert + \left\lVert f^* \right\rVert \left\lVert K(\cdot, X_t)^{\top} \overline{V}_{t-1}^{-1}  \right\rVert
                    \notag
                    \\
                    & \le \left\lVert K(\cdot, X_t)^{\top}\overline{V}_{t-1}^{-1/2} \right\rVert \cdot \left(\left\lVert \overline{V}_{t-1}^{-1/2}  \sum\limits_{s = 1}^{t-1}  \varepsilon_s   K(\cdot, X_s) \right\rVert + D \right) 
                    \quad. 
              \label{temp 101}
          \end{align}
Now, we have
    \begin{equation}
        \begin{array}{lll}
             I_3 & = &  \sum\limits_{t=1}^{T}\E \left[\left(\hat{f}_{t-1}(X_t) - f^*(X_t)\right) \cdot  \bm{1} \left\{E_{t-1} \right\}   \right] \\
             && \\
             & \le &  \sum\limits_{t=1}^{T}\E \left[\left|\hat{f}_{t-1}(X_t) - f^*(X_t)\right| \cdot  \bm{1} \left\{E_{t-1} \right\}   \right] \\
             && \\
& \le^{(a)} &  \sum\limits_{t=1}^{T}\E \left[\left\lVert K(\cdot, X_t)^{\top}\overline{V}_{t-1}^{-1/2} \right\rVert \cdot \left(\left\lVert \overline{V}_{t-1}^{-1/2}  \sum\limits_{s = 1}^{t-1}  \varepsilon_s   K(\cdot, X_s) \right\rVert + D \right)  \cdot  \bm{1} \left\{E_{t-1} \right\}   \right] \\
 && \\
              
             & \le^{(b)}  &   \sum\limits_{t=1}^{T}\E \left[\left\lVert \overline{V}_{t-1}^{-1/2}K(\cdot, X_t) \right\rVert \cdot \left( \sqrt{2R^2 \log \left( \sqrt{\det\left( \overline{V}_{t-1}\right)} /\delta \right)}  + D  \right) \right]  \\
            & \le^{(c)}  &   \left( \sqrt{2R^2 \gamma_T + 2R^2 \ln(1/\delta)}  + D  \right)\E \left[\sum\limits_{t=1}^{T}\left\lVert \overline{V}_{t-1}^{-1/2}K(\cdot, X_t) \right\rVert \right] \\
            & \le^{(d)}  &  \left( \sqrt{2R^2 \gamma_T + 2R^2 \ln(1/\delta)}  + D  \right)\E \left[\sqrt{T \cdot  \sum\limits_{t=1}^{T}\left\lVert \overline{V}_{t-1}^{-1/2}K(\cdot, X_t) \right\rVert^2} \right] \\ & \le^{(e)}  &   \left( \sqrt{2R^2 \gamma_T + 2R^2 \ln(1/\delta)}  + D  \right)\E \left[ \sqrt{T \cdot 2\log \left( \det\left( \overline{V}_{T}\right)  \right)  } \right] \\
            & \le^{(f)} &  \left( \sqrt{2R^2 \gamma_T + 2R^2 \ln(1/\delta)}  + D  \right)\sqrt{4T \gamma_T}\quad.
        \end{array}
    \end{equation}
    Step (a) uses the constructed upper bound shown in (\ref{temp 101}) and step (b) 
 uses the argument that if event $E_{t-1}$ is true, we have $ \left\lVert  \overline{V}_{t-1}^{-1/2}\sum\limits_{s = 1}^{t-1}  \varepsilon_s   K(\cdot, X_s)   \right\rVert \le  \sqrt{2R^2 \log \left( \sqrt{\det\left( \overline{V}_{t-1}\right)} /\delta \right)}$.
 Step (c) uses (\ref{def: RKHS complexity}), stating that  $ \log \left(\det \left(\overline{V}_{t-1}\right) \right)  \le  2\gamma_T $ and step (d) uses Cauchy-Schwarz inequality. Step (e) uses Fact~\ref{lemma: elliptical potential} and step (f) uses $ \log \left(\det \left(\overline{V}_{T}\right) \right)  \le  2\gamma_T $.
 Plugging in $\delta = \frac{1}{2TD^2}$ concludes the proof. \end{proof}

\subsection{Proof of Lemma~\ref{lemma: I_2}}

    \begin{proof}[Proof of Lemma~\ref{lemma: I_2}]
        We have
 \begin{equation}
        \begin{array}{lll}
            I_2 & = &\sum\limits_{t=1}^{T}\E \left[ \tilde{f}_{t}(X_t) -  \hat{f}_{t-1}(X_t) \right] \\
           & = & \sum\limits_{t=1}^{T}\E \left[w_{t}\cdot g_{t-1}(X_t)  \right] \\
            & \le & \E \left[\sum\limits_{t=1}^{T} |w_t| \cdot g_{t-1}(X_t)  \right] \\
            & \le^{(a)} & \E \left[ \mathop{\max}_{t \in [T]} |w_t| \cdot \sum\limits_{t=1}^{T} g_{t-1}(X_t)  \right] \\
            & \le^{(b)} & \E \left[ \mathop{\max}_{t \in [T]} |w_t|   \cdot \sqrt{T \cdot \sum\limits_{t=1}^{T} g^2_{t-1}(X_t)}  \right] \\
            & \le^{(c)} & \E \left[ \mathop{\max}_{t \in [T]} |w_t|   \cdot \sqrt{T \cdot \left( \sqrt{2R^2 \gamma_T + 2R^2\log(2) }  + D  \right)^2  \cdot 4 \gamma_T }  \right] \\
            & = & \sqrt{T \cdot \left( \sqrt{2R^2 \gamma_T + 2R^2\log(2) }  + D  \right)^2  \cdot 4 \gamma_T } \cdot  \E \left[ \mathop{\max}_{t \in [T]} |w_t|      \right] \\
           & =^{(d)} & C_{2,T}  \cdot \left( \sqrt{2R^2 \gamma_T + 2R^2\log(2) }  + D  \right)  \cdot  \sqrt{T \cdot 4\gamma_T} \quad,
        \end{array}
    \end{equation}
    where step (a) uses Holder's inequality, step (b) uses Cauchy-Schwarz inequality, and step (d) uses $C_{2,T} = \E \left[\mathop{\max}\limits_{t \in [T]}|w_{t}| \right]$.
Step (c) uses the fact that 
\begin{equation}
    \begin{array}{lll}
         \sum\limits_{t=1}^{T} g^2_{t-1}(X_t) &= &\sum\limits_{t=1}^{T}  \left( \sqrt{2R^2 \log \left(2 \sqrt{\det\left( \overline{V}_{t-1}\right)}  \right)}  + D  \right)^2 \cdot \left\lVert  \overline{V}_{t-1}^{-1/2} K(\cdot, X_t)\right\rVert^2 \\
         & \le &\left( \sqrt{2R^2 \log \left(2 \sqrt{\det\left( \overline{V}_{T}\right)}  \right)}  + D  \right)^2 \sum\limits_{t=1}^{T}    \left\lVert  \overline{V}_{t-1}^{-1/2} K(\cdot, X_t)\right\rVert^2 \\
      & \le^{(e)} &\left( \sqrt{2R^2 \log \left(2 \sqrt{\det\left( \overline{V}_{T}\right)}  \right)}  + D  \right)^2 \cdot  2 \log \left( \det \left(\overline{ V}_T \right)  \right) \\
      & \le^{(f)} & \left( \sqrt{2R^2 \gamma_T + 2R^2\log(2) }  + D  \right)^2  \cdot 4 \gamma_T\quad,  
    \end{array}
    \label{eq: TMPR}
\end{equation}
where step (e) uses Fact~\ref{lemma: elliptical potential} and step (f) uses $ \log \left(\det \left(\overline{V}_{T}\right) \right)  \le  2\gamma_T $.\end{proof}

\subsection{Proof of Lemma~\ref{lemma: I_1}}
\begin{proof}[Proof of Lemma~\ref{lemma: I_1}]

The first step in our proof will rely on the
following technical lemma, inspired by \citet{russo2019worst}.

 \begin{lemma}
We have
$f^*(x_*) \le \E \left[ \tilde{f}_{t}(X_t) \right]  +    \E \left[\frac{2 |w_t|}{C_{1,t}}  \cdot  g_{t-1}(X_t) \right] + \E \left[ \frac{2|w_t|}{C_{1,t}} \right] \cdot \E \left[   g_{t-1}(X_t) \right]$. 
\label{lemma: appro}
\end{lemma}

Now, we are ready to prove Lemma~\ref{lemma: I_1}. We have 
    \begin{equation}
        \begin{array}{lll}
             I_1 & = & \sum\limits_{t=1}^{T} f^*(x_*) - \E \left[\tilde{f}_t(X_t)  \right] \\
             & \le^{(a)} & \sum\limits_{t=1}^{T}    \E \left[ \frac{2|w_t|}{C_{1,t}}   \cdot  g_{t-1}(X_t) \right] + \sum\limits_{t=1}^{T}  \E \left[ \frac{2|w_t|}{C_{1,t}}   \right] \cdot \E \left[   g_{t-1}(X_t) \right] \\
              & \le^{(b)} & 2 \E \left[ \mathop{\max}_{t \in [T]} |\frac{ w_t}{C_{1,t}}| \cdot \sum\limits_{t=1}^{T} g_{t-1}(X_t)  \right] + 2 \E \left[ \mathop{\max}_{t \in [T]} \frac{ \E \left[   |w_t| \right]}{C_{1,t}}  \cdot \sum\limits_{t=1}^{T}  g_{t-1}(X_t) \right] \\
              & \le^{(c)} & 2\E \left[\mathop{\max}\limits_{t \in [T]     } \frac{|w_t|}{C_{1,t}}  \cdot \sqrt{T \cdot \sum\limits_{t=1}^{T}  g^2_{t-1}(X_t)} \right]   + 2\E \left[ \mathop{\max}_{t \in [T]} \frac{ \E \left[   |w_t| \right]}{C_{1,t}}  \cdot \sqrt{T \cdot \sum\limits_{t=1}^{T}  g^2_{t-1}(X_t)} \right] \\
               & \le^{(d)} & 2\left(\E \left[\mathop{\max}\limits_{t \in [T]     } \frac{|w_t|}{C_{1,t}} \right] + \mathop{\max}\limits_{t \in [T]     } \E \left[\frac{|w_t|}{C_{1,t}} \right]\right) \cdot  \left( \sqrt{2R^2 \gamma_T + 2R^2\log(2) }  + D  \right)  \cdot  \sqrt{T \cdot 4\gamma_T} \\
               
              & = & 2 C_{3,T}  \cdot  \left( \sqrt{2R^2 \gamma_T+ 2R^2\log(2) }  + D  \right)  \cdot  \sqrt{T \cdot 4\gamma_T}\quad,
    
        \end{array}
    \end{equation}
    where step (a) uses Lemma~\ref{lemma: appro}, step (b) uses Holder's inequality, and step (c) uses Cauchy-Schwarz inequality. Step (d) uses (\ref{eq: TMPR}).
\end{proof}


\begin{proof}[Proof of Lemma~\ref{lemma: appro}]

Let  $r(t) := f^*(x_*)  - \E\left[ \tilde{f}_t(X_t)   \right] =  \E \left[f^*(x_*)  - \mathop{\max}\limits_{x \in \mathcal{X}}\tilde{f}_t(x)  \right] $. 
We construct an upper bound for $r(t)$, 
using the following two steps. 
\begin{enumerate}
    \item By introducing the activation function $(x)^+ := \max \{0,x\}$, we use Markov's inequality to show
 \begin{equation}
 \begin{array}{lll}
   r(t) &\le&  \frac{2}{C_{1,t}} \cdot \E \left[ \left( \mathop{\max}\limits_{x \in \mathcal{X}}\tilde{f}_t(x)  - \E \left[\mathop{\max}\limits_{x \in \mathcal{X}}\tilde{f}_t(x)  \right] \right)^{+}\right] \quad.
      \end{array}
     \label{eq: markov}
 \end{equation}

 \item We use a ``ghost sample'' to show 
    \begin{equation}
    \begin{array}{lll}
      \E \left[ \left( \mathop{\max}\limits_{x \in \mathcal{X}}\tilde{f}_t(x)  - \E \left[\mathop{\max}\limits_{x \in \mathcal{X}}\tilde{f}_t(x)  \right] \right)^{+}\right] &\le &  \E \left[ |w_t| \cdot  g_{t-1}(X_t) \right] +  \E \left[  | w_t | \right] \cdot \E \left[   g_{t-1}(X_t) \right] \quad.
       \end{array}
       \label{eq: ghost}
    \end{equation}
\end{enumerate}

    Combining the above two steps
 gives
\begin{equation}
  f^*(x_*) \le \E \left[ \tilde{f}_{t}(X_t) \right]  +    \E \left[ \frac{2|w_t|}{C_{1,t}}   \cdot  g_{t-1}(X_t) \right] +  \E \left[ \frac{2 |w_t|}{C_{1,t}}  \right] \cdot \E \left[   g_{t-1}(X_t) \right] \quad,
\end{equation}
which will conclude the proof.

    \paragraph{Proof of (\ref{eq: markov}):} If $r(t) <0$, inequality (\ref{eq: markov}) trivially holds. If $r(t) >0$, we use Markov's inequality and have
    \begin{equation}
    \begin{array}{lll}
  \mathbb{P} \left\{  \left( \mathop{\max}\limits_{x \in \mathcal{X}}\tilde{f}_t(x)  - \E \left[\mathop{\max}\limits_{x \in \mathcal{X}}\tilde{f}_t(x)  \right] \right)^{+} \ge r(t) \right\}  &\le  & \E \left[ \left( \mathop{\max}\limits_{x \in \mathcal{X}}\tilde{f}_t(x)  - \E \left[\mathop{\max}\limits_{x \in \mathcal{X}}\tilde{f}_t(x)  \right] \right)^{+}\right]/r(t)\quad,
 \end{array}
 \end{equation}
 which gives
 \begin{equation}
     \begin{array}{lllll}
            r(t) & \le &\frac{\E \left[ \left( \mathop{\max}\limits_{x \in \mathcal{X}}\tilde{f}_t(x)  - \E \left[\mathop{\max}\limits_{x \in \mathcal{X}}\tilde{f}_t(x)  \right] \right)^{+}\right]}{\mathbb{P} \left\{  \left( \mathop{\max}\limits_{x \in \mathcal{X}}\tilde{f}_t(x)  - \E \left[\mathop{\max}\limits_{x \in \mathcal{X}}\tilde{f}_t(x)  \right] \right)^{+} \ge r(t) \right\} }
            & =& \frac{\E \left[ \left( \mathop{\max}\limits_{x \in \mathcal{X}}\tilde{f}_t(x)  - \E \left[\mathop{\max}\limits_{x \in \mathcal{X}}\tilde{f}_t(x)  \right] \right)^{+}\right]}{\mathbb{P} \left\{  \left( \mathop{\max}\limits_{x \in \mathcal{X}}\tilde{f}_t(x)  - \E \left[\mathop{\max}\limits_{x \in \mathcal{X}}\tilde{f}_t(x)  \right] \right)^{+} \ge f^*(x_*)  - \E \left[ \mathop{\max}\limits_{x \in \mathcal{X}}\tilde{f}_t(x)   \right]  \right\} }\\
                &&  & \le &\frac{\E \left[ \left( \mathop{\max}\limits_{x \in \mathcal{X}}\tilde{f}_t(x)  - \E \left[\mathop{\max}\limits_{x \in \mathcal{X}}\tilde{f}_t(x)  \right] \right)^{+}\right]}{\mathbb{P} \left\{  \mathop{\max}\limits_{x \in \mathcal{X}}\tilde{f}_t(x)  - \E \left[\mathop{\max}\limits_{x \in \mathcal{X}}\tilde{f}_t(x)  \right]  \ge f^*(x_*)  - \E \left[ \mathop{\max}\limits_{x \in \mathcal{X}}\tilde{f}_t(x)  \right]  \right\} }\\
        && & = &\frac{\E \left[ \left( \mathop{\max}\limits_{x \in \mathcal{X}}\tilde{f}_t(x)  - \E \left[\mathop{\max}\limits_{x \in \mathcal{X}}\tilde{f}_t(x)  \right] \right)^{+}\right]}{\mathbb{P} \left\{  \mathop{\max}\limits_{x \in \mathcal{X}}\tilde{f}_t(x)   \ge f^*(x_*)    \right\} }\\
        && & \le & \frac{2}{C_{1,t}} \cdot \E \left[ \left( \mathop{\max}\limits_{x \in \mathcal{X}}\tilde{f}_t(x)  - \E \left[\mathop{\max}\limits_{x \in \mathcal{X}}\tilde{f}_t(x)  \right] \right)^{+}\right]\quad,
            
         \end{array}
    \end{equation}
    where the last step uses Lemma~\ref{lemma: anti} below.
      \begin{lemma}
For any round $t \ge 1$, we have
$\mathbb{P} \left\{  \mathop{\max}\limits_{x \in \mathcal{X}}\tilde{f}_t(x)   \ge f^*(x_*)    \right\}  \ge 0.5 C_{1,t}$.
        \label{lemma: anti}
    \end{lemma}

    \paragraph{Proof of (\ref{eq: ghost}):}
Let $\tilde{w}_t \sim P_{w,t}$ be an independent copy of $w_t$. We construct another random function $\tilde{\tilde{f}}_t(x) =  \hat{f}_{t-1}(x)  +  \tilde{w}_t \cdot   g_{t-1}(x)$.
We have $\E \left[\mathop{\max}\limits_{x \in \mathcal{X}}\tilde{f}_t(x)  \right]  =  \E \left[\E \left[\mathop{\max}\limits_{x \in \mathcal{X}}\tilde{f}_t(x)  \mid \mathcal{F}_{t-1} \right] \right] = \E \left[\E \left[\mathop{\max}\limits_{x \in \mathcal{X}}\tilde{\tilde{f}}_t(x) \mid \mathcal{F}_{t-1} \right] \right] = \E \left[\mathop{\max}\limits_{x \in \mathcal{X}}\tilde{\tilde{f}}_t(x)  \right] $ due to the fact that given $\mathcal{F}_{t-1}$, random functions $\tilde{f}_t $ and $\tilde{\tilde{f}}_t$ have the same distribution. 

\begingroup\allowdisplaybreaks
Now, we have 
\begin{equation}
    \begin{array}{lll}
     \E \left[ \left( \mathop{\max}\limits_{x \in \mathcal{X}}\tilde{f}_t(x)  - \E \left[\mathop{\max}\limits_{x \in \mathcal{X}}\tilde{f}_t(x)  \right] \right)^{+}\right] & = &  \E \left[ \left( \mathop{\max}\limits_{x \in \mathcal{X}}\tilde{f}_t(x)  - \E \left[\mathop{\max}\limits_{x \in \mathcal{X}}\tilde{\tilde{f}}_t(x)  \right] \right)^{+}\right] \\
      & = &   \E \left[ \left( \tilde{f}_t(X_t)  - \E \left[\mathop{\max}\limits_{x \in \mathcal{X}}\tilde{\tilde{f}}_t(x)  \right] \right)^{+}\right]   \\
      & = &   \E \left[ \left( \tilde{f}_t(x_t)  - \E \left[\mathop{\max}\limits_{x \in \mathcal{X}}\tilde{\tilde{f}}_t(x)\mid X_t, \tilde{f}_t  \right] \right)^{+}\right]    \\
      & = &  \E \left[ \left( \E \left[\tilde{f}_t(X_t)  -\mathop{\max}\limits_{x \in \mathcal{X}}\tilde{\tilde{f}}_t(x)\mid X_t, \tilde{f}_t \right] \right)^{+}\right]    \\
       & \le &   \E \left[ \left( \E \left[\tilde{f}_t(X_t)  -\tilde{\tilde{f}}_t(X_t)\mid X_t, \tilde{f}_t \right] \right)^{+}\right]    \\
         & \le & \E \left[ \left| \E \left[\tilde{f}_t(X_t)  -\tilde{\tilde{f}}_t(X_t)\mid X_t, \tilde{f}_t \right] \right| \right]    \\
              & \le & \E \left[  \E \left[\left| \tilde{f}_t(X_t)  -\tilde{\tilde{f}}_t(X_t)\right|\mid X_t, \tilde{f}_t \right]  \right]    \\
                  & = &  \E \left[  \left| \tilde{f}_t(X_t)  -\tilde{\tilde{f}}_t(X_t)\right|  \right]    \\
                     & = &  \E \left[  \left| \tilde{f}_t(X_t)  -\tilde{\tilde{f}}_t(X_t) - \hat{f}_{t-1}(X_t) + \hat{f}_{t-1}(X_t)\right|  \right]    \\
                      & \le &  \E \left[  \left| \tilde{f}_t(X_t)  - \hat{f}_{t-1}(X_t) \right|  \right]+ \E \left[  \left| \tilde{\tilde{f}}_t(X_t) - \hat{f}_{t-1}(X_t) \right|  \right]    \\
                       & \le &  \E \left[|w_t \cdot  g_{t-1}(X_t)| \right] +  \E \left[|\tilde{w}_t \cdot g_{t-1}(X_t) |\right]     \\

     & \le &  \E \left[|w_t| \cdot  g_{t-1}(X_t) \right] + \E \left[|w_t| \right]      \cdot \E \left[g_{t-1}(X_t) \right]     \quad.
    \end{array}
\end{equation}
\endgroup


\end{proof}

    \begin{proof}[Proof of Lemma~\ref{lemma: anti}]
Let $h := K(\cdot, x_*) $ be a fixed but (possibly infinite-dimensional) real-valued vector. 
      Define $G_{t-1}$ as the event that  
                $  \left\lVert  \overline{V}_{t-1}^{-1/2}\sum\limits_{s = 1}^{t-1}  \varepsilon_s   K(\cdot, X_s)   \right\rVert \le  \sqrt{2R^2 \log \left(2 \sqrt{\det\left( \overline{V}_{t-1}\right)}  \right)} $. Then, from Theorem~\ref{concentration}, for any $t \ge 1$, we have $\mathbb{P}\left\{G_{t-1}  \right\} \ge 0.5$. 

 We have
\begin{equation}
    \begin{array}{lll}
       \hat{f}_{t-1}(x_*)  - f^*(x_*) 

                   & = &  \left\langle  \overline{V}_{t-1}^{-1}  \sum\limits_{s = 1}^{t-1}  \varepsilon_s   K(\cdot, X_s),  h \right\rangle -  \left\langle    \overline{V}_{t-1}^{-1}  f^*,  h \right\rangle \quad.
              \end{array}
          \end{equation}
          If event $G_{t-1}$ holds, we have
          \begin{equation}
          \begin{array}{lll}
            \left\lvert \hat{f}_{t-1}(x_*)  - f^*(x_*) \right\rvert & \le & \left\lVert \overline{V}_{t-1}^{-1/2}h \right\rVert \cdot \left\lVert \overline{V}_{t-1}^{-1/2}  \sum\limits_{s = 1}^{t-1}  \varepsilon_s   K(\cdot, X_s) \right\rVert +  \left\lVert  h^{\top}\overline{V}_{t-1}^{-1} f^*   \right\rVert \\
           & \le & \left\lVert \overline{V}_{t-1}^{-1/2}h \right\rVert \cdot \sqrt{2R^2 \log \left(2 \sqrt{\det\left( \overline{V}_{t-1}\right)}  \right)}  + \left\lVert  \overline{V}_{t-1}^{-1/2}  f^* \right\rVert \\
            & \le & \left\lVert \overline{V}_{t-1}^{-1/2}h \right\rVert \cdot \left( \sqrt{2R^2 \log \left(2 \sqrt{\det\left( \overline{V}_{t-1}\right)}  \right)}  + \left\lVert f^* \right\rVert  \right) \\
             & \le & \left\lVert \overline{V}_{t-1}^{-1/2}h \right\rVert \cdot \left( \sqrt{2R^2 \log \left(2 \sqrt{\det\left( \overline{V}_{t-1}\right)}  \right)}  + D  \right) \\
             & = & g_{t-1}(x_*)\quad.
               \end{array}
          \end{equation}

Therefore,    we have
\begin{align*}
  \mathbb{P} \left\{  \mathop{\max}\limits_{x \in \mathcal{X}}\tilde{f}_t(x)   \ge f^*(x_*)    \right\}
 &\ge \mathbb{P} \left\{{G}_{t-1} \right\} \cdot \mathbb{P} \left\{  \mathop{\max}\limits_{x \in \mathcal{X}}\tilde{f}_t(x)   \ge f^*(x_*)  \mid  {G}_{t-1}  \right\} \\
  &\ge 0.5 \cdot \mathbb{P} \left\{  \tilde{f}_t(x_*)   \ge f^*(x_*)  \mid  {G}_{t-1}  \right\} \\
&= 0.5 \cdot \mathbb{P} \left\{  \tilde{f}_t(x_*) - \hat{f}_{t-1}(x_*)   \ge f^*(x_*) -  \hat{f}_{t-1}(x_*)   \mid  {G}_{t-1}  \right\} \\
  &= 0.5\cdot \mathbb{P} \left\{  w_t \cdot g_{t-1}(x_*)   \ge f^*(x_*) -  \hat{f}_{t-1}(x_*)   \mid  {G}_{t-1}  \right\} \\
    &\ge 0.5\cdot \mathbb{P} \left\{  w_t \cdot g_{t-1}(x_*)   \ge \left| f^*(x_*) -  \hat{f}_{t-1}(x_*) \right|   \mid  {G}_{t-1}  \right\} \\
  &\ge  0.5 \cdot \mathbb{P} \left\{  w_t \cdot g_{t-1}(x_*)   \ge   g_{t-1}(x_*)   \mid  {G}_{t-1}  \right\}  \\
&=  0.5\cdot \mathbb{P} \left\{  w_t    \ge  1 \mid  {G}_{t-1}  \right\}  \\
  &\ge 0.5 \cdot C_{1,t}\quad,
        \end{align*} 
        which concludes the proof. \end{proof}

        \section{Proof of Theorem~\ref{theorem: linar bandit bound}} \label{app: linar bandit bound}

    \thmlinearbandit*

       \subsection{Upper Bound Proof}  
       \begin{proof}
           Recall that the linear kernel $K(x, x') = x^{\top}x'$ has kernel complexity $\gamma_T \le  O (d \log(T))$. Plugging the upper bound of $\gamma_T$ in Theorem~\ref{theorem: expected regret} concludes the proof. \end{proof}

       \subsection{Lower Bound Proof}
       \begin{proof}
The main challenge for the lower bound proof is that all $x \in \mathcal{X}$ share the same $w_t$ to do exploration. Therefore, we cannot reuse the algorithm-dependent regret lower bound analysis from \citet{agrawalnear}. There are a few key steps for the lower bound proof.
\begin{enumerate}
    \item Construct a linear bandit optimization problem instance.
    \item Rewrite the learning algorithm.
    \item Lower bound the probability of pulling a sub-optimal arm when $t $ is large enough.    
\end{enumerate}

\paragraph{Construct a learning problem instance.}
        Let $\mathcal{X} = \left\{e_i: 1 \le i \le d \right\}$ be a standard basis of $\mathbb{R}^d$. 
        We construct 
$f^* = (\Delta, 0, \dotsc, 0)$, a $d$-dimensional vector, where $\Delta$ is in the order of $\frac{\sqrt{R^2 d^2} + D\sqrt{d}}{\sqrt{T}}$ with the details specified later. Note that 
 $ e_1 = (1, 0, \dotsc, 0)$ is the optimal arm and any $ e_i = (0, \dotsc,0, \underbrace{1}_{i^{th}}, 0,\dotsc, 0)$ is  sub-optimal.
We have $f^*(e_1) = \Delta$ and $f^*(e_i) = 0$ for all $i \ne 1$, i.e.,  $\Delta = f^*(e_1) - f^*(e_i)$ for all $i \ne 1$. 
We construct $\varepsilon_t \equiv 0$ for all $t$, i.e.,  $Y_t = f^*(X_t)$. 

\paragraph{Rewrite the learning algorithm.}
Let $Q_{i}(t)$ denote the total number of times that $e_i \in \mathcal{X}$ has been played by the end of round $t$.

Recall $\overline{V}_t = I_H + \sum\limits_{s=1}^t K(\cdot, X_s) K(\cdot, X_s)^{\top} = \sum\limits_{s=1}^t X_s X_s^{\top}$. This matrix is a diagonal matrix for our constructed problem instance. We have 
\begin{equation}
    \overline{V}_{t}  =
  \begin{bmatrix}
  1+ Q_{1}(t) & & \\
    & \ddots & \\
    & & 1+ Q_{d}(t)
  \end{bmatrix}_{d \times d}, \quad \overline{V}_{t}^{-1}  =
  \begin{bmatrix}
   \frac{1}{1+ Q_{1}(t)} & & \\
    & \ddots & \\
    & & \frac{1}{1+ Q_{d}(t)}
  \end{bmatrix}_{d \times d}\quad.
\end{equation}
The determinant of $\overline{V}_{t}$ is $\det\left( \overline{V}_{t} \right) = \prod_{i \in [d]} (1+ Q_i(t))$.

We rewrite the empirical function
\begin{equation}
\begin{array}{lll}
  \hat{f}_{t-1} &= &\overline{V}_{t-1}^{-1} \sum\limits_{s = 1}^{t-1} Y_s X_s \\
  
 & =& \overline{V}_{t-1}^{-1} \sum\limits_{s = 1}^{t-1}(f^{*}(X_s) + \varepsilon_s)X_s \\
  &=& \overline{V}_{t-1}^{-1} \left(\sum\limits_{s = 1}^{t-1}f^{*}(X_s)X_s \cdot \bm{1}\left\{X_s = e_1 \right\} + \sum\limits_{s = 1}^{t-1}f^{*}(X_s)X_s \cdot \bm{1}\left\{X_s \ne e_1 \right\} \right)\\
    &=& \overline{V}_{t-1}^{-1} f^{*}(e_1)e_1 \sum\limits_{s = 1}^{t-1}\bm{1}\left\{X_s = e_1 \right\} \\
  & = &   \begin{bmatrix}
   \frac{1}{1+ Q_{1}(t-1)} & & \\
    & \ddots & \\
    & & \frac{1}{1+ Q_{d}(t-1)}
  \end{bmatrix} \Delta \cdot Q_{1}(t-1) \cdot  e_1 \\
  && \\
  &= & \frac{\Delta \cdot Q_{1}(t-1)}{1+ Q_{1}(t-1)} e_1 \quad.
   
    \end{array}
\end{equation}

Recall $g_{t-1}(x) = \left(\sqrt{2R^2 \log \left(2 \sqrt{ \det \left(\overline{V}_{t-1} \right)}  \right)} + D\right)\cdot \left\lVert  K(\cdot, x)^{\top}\overline{V}_{t-1}^{-1/2}  \right\rVert$.

Let $D_t := \sqrt{2R^2 \log \left(2 \sqrt{ \det \left(\overline{V}_{t-1} \right)}  \right)} = \sqrt{2R^2 \log  \left(2 \sqrt{\prod_{i \in [d]} (1+ Q_{i}(t))} \right)} $.

Then, we have
$g_{t-1}(e_i) = \left(D_{t-1} + D \right)
\cdot \left\lVert  e_i^{\top}\overline{V}_{t-1}^{-1/2}  \right\rVert = \left(D_{t-1} + D \right)  \cdot \sqrt{\frac{1}{1+ Q_{i}(t-1)}}$, which gives  

\begin{equation}
\begin{array}{lll}
\tilde{f}_t(e_1) &= & \frac{\Delta \cdot Q_{1}(t-1)}{1+ Q_{1}(t-1)}  + w_{t} \cdot \left(D_{t-1} + D \right)  \cdot \sqrt{\frac{1}{1+ Q_{1}(t-1)}}\quad, \\
&& \\
\tilde{f}_t(e_i) &= & 0  + w_{t}\cdot \left(D_{t-1}+ D \right)  \cdot \sqrt{\frac{1}{1+ Q_{i}(t-1)}}, \quad \forall i \ne 1\quad. 
\end{array}
\end{equation}


\paragraph{Lower bound the probability of pulling sub-optimal arms.}  Fix $\beta = \frac{1}{6}$ and $\alpha = \frac{5.1}{6}$. We have
\begin{equation}
\begin{array}{lll}
    \mathcal{R}(T)  =  \sum\limits_{t = 1}^{\beta T} \E \left[\bm{1} \left\{X_t \ne e_1 \right\} \right] \Delta    + \sum\limits_{t = \beta T +1}^{ T} \E \left[\bm{1} \left\{X_t \ne e_1 \right\} \right] \Delta 
    \ge  \sum\limits_{t = \beta T +1}^{ T} \E \left[\bm{1} \left\{X_t \ne e_1 \right\} \right] \Delta  \quad.
\end{array}
\end{equation}
Now, we  construct a lower bound for the regret from round $t = \beta T +1$ to round $T$.
Define $B_t^* := \left\{Q_1(t) > t- \beta \cdot T \right\}$ as the event that by the end of round $t$, the optimal arm has been played at least $t - \beta \cdot T$ times. 
We lower bound the total regret from round $t= \beta T + 1$ to the end of round $T$ by analyzing two cases separately based on the events of $B_t^*$ for all rounds $t \in [\beta T +1, T]$: 
\textbf{events $B_t^*$ for all $t \in [\beta T + 1, T]$ are true;}
 \textbf{there exists $t \in \left[ \beta T + 1, T\right]$ such that $B_t^*$ is false.}

\paragraph{Events $B_t^*$ for all $t \in [\beta T + 1, T]$ are true.}

We further consider the total regret from round $ t  = \alpha T+1 $ to round $T$ (the blue part in Figure~\ref{lower bound}). Note that the total regret from round $1$ to round $T$ is lower bounded by the total regret from round $\alpha T+1 $ to round $T$.

 Before computing the probability of playing a sub-optimal arm in any round $t \ge \alpha T+1$, we construct the following bounds on the number of pulls first.

 If $B_{t-1}^*$ is true, we have
 \begin{equation}
     \begin{array}{l}
          Q_1(t-1) > t-1-\beta T \ge \alpha T+1 - 1 -\beta T = (\alpha - \beta)T\quad, 
     \end{array}
          \label{temp 1}
 \end{equation}
 which lower bounds the total number of pulls of the optimal arm by the end of round $t-1$. 
 
 The above lower bound also implies the following upper bounds on the number of pulls of the sub-optimal arms. We have
 \begin{equation}
     \begin{array}{l}
           \sum\limits_{i \ne 1}Q_i(t-1) = t -1 - Q_1(t-1) \le t-1 -(t-1 -\beta T) = \beta T\quad,      
     \end{array}
     \label{ttt3}
 \end{equation}
 and
 \begin{equation}
     \begin{array}{l}
            Q_i(t-1) \le \beta T, \quad \forall i \ne 1\quad. 
     \end{array}
     \label{temp 3}
 \end{equation}
From (\ref{temp 1}) and (\ref{temp 3}), for any $i \ne 1$, we have
\begin{equation}
    \sqrt{\frac{Q_i(t-1) +1}{Q_1(t-1)+1}} \le \sqrt{\frac{\beta T +1 }{(\alpha - \beta)T +1}} \le  \frac{1}{2} \quad \Rightarrow 1 - \sqrt{\frac{Q_i(t-1) +1}{Q_1(t-1)+1}} \ge \frac{1}{2}.
    \label{ttt2}
\end{equation}

When $T > \frac{e^d}{(\alpha-\beta)}$ is large enough, we  have
\[ \log  \left(2 \sqrt{\prod_{i \in [d]} (1+ Q_{i}(t-1))} \right)  \ge 0.5 \log  \left( (1+ Q_{1}(t-1)) \right) \ge 0.5 \log ((\alpha-\beta)T) > 0.5d \quad,\] which gives
\begin{equation}
   D_{t-1} = \sqrt{2R^2 \log  \left(2 \sqrt{\prod_{i \in [d]} (1+ Q_{i}(t-1))} \right)} > R \sqrt{0.5d}\quad.
   \label{ttt}
\end{equation}


Let $ \Delta =\frac{(R \sqrt{0.5d} + D)}{2 \sqrt{1+ \beta T/(d-1)}} $. Now, we are ready to lower bound the 
probability of pulling a sub-optimal arm in round $t \ge \alpha T + 1$ given $B_{t-1}^{*}$ is true. We have
\begin{equation}
    \begin{array}{ll}
   &  \mathbb{P} \left\{\exists i \ne 1:  X_t = e_i \mid B_{t-1}^*\right\} \\
  \ge    &   \mathbb{P} \left\{\exists i \ne 1:  \tilde{f}_t(e_i) > \tilde{f}_t(e_1) \mid B_{t-1}^*\right\}  \\
      = &   \mathbb{P} \left\{\exists i \ne 1:  w_{t}  \cdot (D_{t-1} + D)  \cdot \sqrt{\frac{1}{1+ Q_{i}(t-1)}} > \frac{\Delta \cdot Q_{1}(t-1)}{1+ Q_{1}(t-1)}  + w_{t} \cdot (D_{t-1} + D) \cdot\sqrt{\frac{1}{1+ Q_{1}(t-1)}}  \mid B_{t-1}^*\right\}  \\
      & \\
      >^{(a)} &   \mathbb{P} \left\{\exists i \ne 1:  w_{t}  \cdot (D_{t-1} + D)  \cdot  \left( \sqrt{\frac{1}{1+ Q_{i}(t-1)}} -  \sqrt{\frac{1}{1+ Q_{1}(t-1)}}\right) \ge \Delta     \mid B_{t-1}^*\right\}  \\

         >^{(b)} &   \mathbb{P} \left\{\exists i \ne 1:  w_{t}  \cdot (R \sqrt{0.5d} + D)  \cdot  \left( \sqrt{\frac{1}{1+ Q_{i}(t-1)}} -  \sqrt{\frac{1}{1+ Q_{1}(t-1)}}\right) \ge \Delta     \mid B_{t-1}^*\right\}  \\
           = &   \mathbb{P} \left\{\exists i \ne 1:  w_{t}  \cdot (R \sqrt{0.5d} + D)  \cdot  \left( 1 -  \sqrt{\frac{1+ Q_{i}(t-1)}{1+ Q_{1}(t-1)}}\right) \ge \Delta \sqrt{1+ Q_{i}(t-1)}    \mid B_{t-1}^*\right\}  \\
           \ge^{(c)} &   \mathbb{P} \left\{\exists i \ne 1:  w_{t}  \cdot (R \sqrt{0.5d} + D)  \cdot  \frac{1}{2}\ge \Delta \sqrt{1+ Q_{i}(t-1)}    \mid B_{t-1}^*\right\}  \\
           
           = &  1- \mathbb{P} \left\{\forall i \ne 1:  w_{t}  \cdot (R \sqrt{0.5d} + D)  \cdot  \frac{1}{2} < \Delta \sqrt{1+ Q_{i}(t-1)}    \mid B_{t-1}^*\right\}  \\
           
            \ge^{(d)} & 1- \mathop{\max}\limits_{z_2, \dotsc, z_d: \sum\limits_{i=2}^d z_i \le \beta T }\mathbb{P} \left\{ \forall i \ne 1: w_{t}  \cdot (R \sqrt{0.5d} + D)  \cdot  \frac{1}{2} < \Delta \sqrt{1+ z_i}   \right\}  \\
           =^{(e)} & 1- \mathbb{P} \left\{ w_{t}  \cdot (R \sqrt{0.5d} + D)  \cdot  \frac{1}{2} < \Delta \sqrt{1+ \beta T/(d-1)}    \right\}  \\
           = & 1- \mathbb{P} \left\{ w_{t}    < \frac{2\Delta \sqrt{1+ \beta T/(d-1)}}{(R \sqrt{0.5d} + D)}    \right\}  \\
             = &  \mathbb{P} \left\{ w_{t}    \ge 1   \right\}  \\
             = & C_{1,t}\quad.
    \end{array}
\end{equation}
The total regret from round $\alpha T +1$ to round $T$ is lower bounded by $(1-\alpha)T \cdot \Delta \cdot \mathop{\min}_{t \in [T]}C_{1,t} = \Omega \left( Rd\sqrt{T} + D\sqrt{dT}\right)$, which lower bounds the regret from round $1$ to round $T$. Now, we provide explanation for the key steps above. Step (a) uses the fact that $\frac{\Delta Q_1(t-1)}{1+Q_1{t-1}} < \Delta$, step (b) uses the lower bound constructed in (\ref{ttt}), and step (c) uses the lower bound constructed in (\ref{ttt2}).

For steps (d) and (e), we construct an auxiliary optimization problem with the constraint shown in (\ref{ttt3}). We have 
\begin{equation}
    \begin{array}{ll}
         \mathop{\max}\limits_{z_2, \dotsc, z_d}  & \mathbb{P} \left\{\forall i \ne 1:  w_{t}  \cdot (R \sqrt{0.5d} + D)  \cdot  \frac{1}{2} < \Delta \sqrt{1 + z_i}   \right\} , \\

        & \\

         s.t. &  \sum_{i = 2}^d z_i \le \beta T. 
    \end{array}
\end{equation}
It is not hard to verify that $z_i = \frac{\beta T}{d-1}$ for all $i \ne 1$ is the optimal solution and step (e) uses this fact.

  \paragraph{There exists $t \in \left[ \beta T + 1, T\right]$ such that $B_t^*$ is false.} Without loss of generality, let us say $t_0$ is the first round such that $B_{t_0}^*$ is false, i.e., we have $ Q_1(t_0) \le t_0- \beta \cdot T$. Then, we know the total number of times that the sub-optimal arms have been played by the end of round $t_0$ is at least
\begin{equation}
    \sum\limits_{i \ne 1} Q_i(t_0) \ge t_0 - (t_0- \beta \cdot T) = \beta T \quad.
\end{equation}

The total regret  from round $1$ to round $t_0$ is at least $\Delta \beta T = \Omega \left( Rd\sqrt{T} + D\sqrt{dT}\right)$, which  lower bounds  
the total regret from round $1$ to round $T$ (the pink part in Figure~\ref{lower bound}). 
           \begin{figure}[!t]
    \begin{center}
    \includegraphics[trim=0 180 0 180,clip,width=\textwidth]{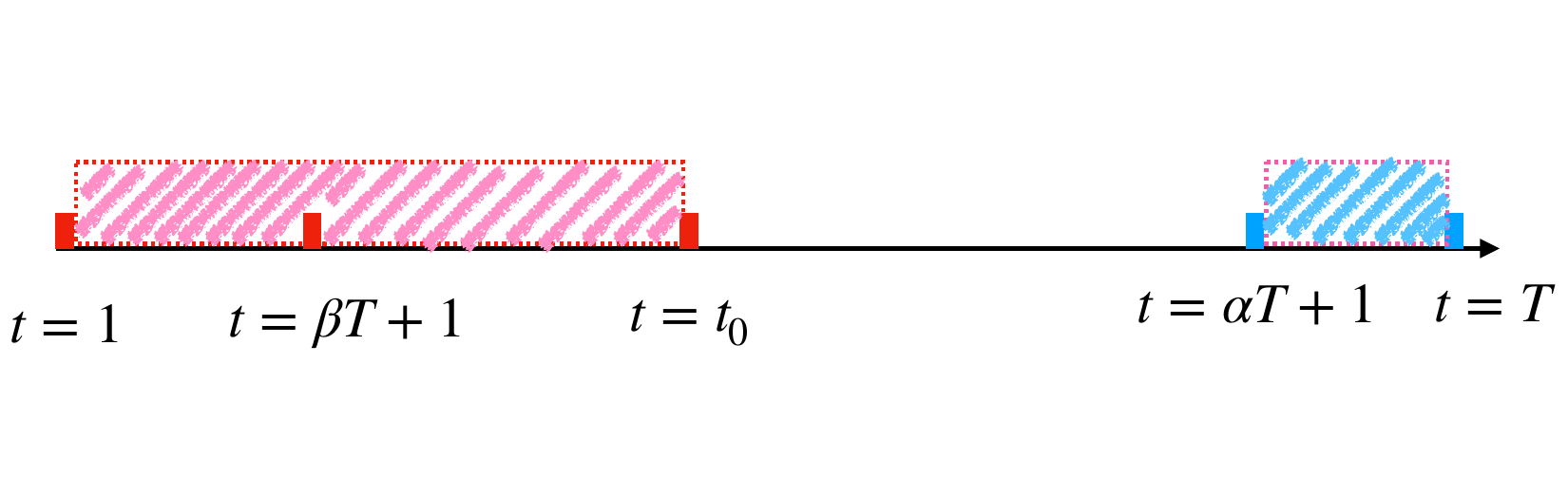}
    \caption{For the case where events $B_t^*$ for all $t \in [\beta T + 1, T]$ are true, we lower bound the regret for all the rounds in the blue color; for the case where there exists $t \in \left[ \beta T + 1, T\right]$ such that $B_t^*$ is false, we lower bound the regret for all the rounds in the pink color.}
    \label{lower bound}
    \end{center}
\end{figure}\end{proof}

        \section{Proof of Theorem~\ref{coro: SG}} \label{app: SG}
        \thmsubgauss*
        \begin{proof}[Proof of Theorem~\ref{coro: SG}]
            
Recall $\E \left[\mathop{\max}\limits_{s \in [t]}|w_{s}| \right] = C_{2,t} ;  \left(\E \left[\mathop{\max}\limits_{s \in [t]     } \frac{|w_s|}{C_{1,s}} \right] + \mathop{\max}\limits_{s \in [t]     } \E \left[\frac{|w_s|}{C_{1,s}} \right]\right) = C_{3,t} $.

 Now,   we compute upper bounds for $C_{2,T}$ and $C_{3,T}$ based on Fact~\ref{fact: maximal}, the maximal inequality of multiple sub-Gaussian random variables. We have
    \begin{equation}
        \begin{array}{l}
          C_{2,T} 
      =   \E_{w_1 \sim P_{w,1}, \dotsc, w_T \sim P_{w,T}} \left[\mathop{\max}\limits_{t \in [T]}|w_{t}| \right] \le \mathop{\max}\limits_{t \in [T]} \sigma_t  \cdot \sqrt{2 \log(2T)}
       
        \end{array}
    \end{equation}
   and 
    \begin{equation}
        \begin{array}{lll}
             C_{3,T} & = &  \E \left[\mathop{\max}\limits_{t \in [T]     } \frac{|w_t|}{C_{1,t}} \right] + \mathop{\max}\limits_{t \in [T]     } \E \left[\frac{|w_t|}{C_{1,t}} \right] \\ 
             & \le & \frac{1}{\mathop{\min}\limits_{t \in [T]}C_{1,t}} \cdot \left( \E \left[\mathop{\max}\limits_{t \in [T]     } |w_t| \right] + \mathop{\max}\limits_{t \in [T]     } \E \left[|w_t| \right] \right) \\
              & \le & \frac{1}{\mathop{\min}\limits_{t \in [T]}C_{1,t}} \cdot \left( \mathop{\max}\limits_{t \in [T]} \sigma_t  \cdot \sqrt{2 \log(2T)} +  \mathop{\max}\limits_{t \in [T]} \sigma_t  \cdot \sqrt{2 \log(2)}  \right) \\
              & = &   O \left( \frac{1}{\mathop{\min}\limits_{t \in [T]}C_{1,t}} \cdot \log(T) \cdot \mathop{\max}\limits_{t \in [T]} \sigma_t  \right)\quad.
        \end{array}
    \end{equation}
    Plugging in the upper bounds for $C_{2,T}$ and $C_{3,T}$ into Theorem~\ref{theorem: expected regret} concludes the proof.  \end{proof}

        \section{Proof of Theorem~\ref{coro: mixture}} \label{app: mixture}
    \thmmixture*    
    \begin{proof}[Proof of Theorem~\ref{coro: mixture}] We first compute a lower bound for $C_{1,t}$ and have $C_{1,t} = \mathbb{P}_{w_t \sim P_{w,t}} \left\{w \ge 1 \right\}  = \rho_t + (1-\rho_t) \cdot \mathbb{P}_{w_t \sim \text{sub-Gaussian}(\sigma_t^2) } \left\{w_t \ge 1 \right\} \ge \rho_t$. Let $[B] \subseteq [T]$ be a random subset of $[T]$ such that in each round $t \in [B]$, we sample $w_t \sim \Ber(1)$.
    Now, we compute an upper bound for $C_{2,T}$. We have
\begin{equation}
    \begin{array}{lll}
       C_{2,T} &
      = &  \E_{w_1 \sim P_{w,1}, \dotsc, w_T \sim P_{w,T}} \left[\mathop{\max}\limits_{t \in [T]}|w_{t}| \right]  \\
      & = & \E \left[\E\left[\mathop{\max}\limits_{t \in [T]}|w_{t}| \mid [B]\right]  \right] \\
        & \le^{(a)} & \E \left[\E\left[\mathop{\max}\limits_{t \in [B]}|w_{t}| + \mathop{\max}\limits_{t \in [T] \setminus [B]}|w_{t}| \mid [B]\right]  \right] \\
         & \le & \E \left[\E\left[1 + \mathop{\max}\limits_{t \in [T] }|w_{t}| \mid [B]\right]  \right] \\
         & = & 1+ \E\left[\mathop{\max}\limits_{t \in [T] }|w_{t}|   \right] \\
 &  \le^{(b)} &  1+  \mathop{\max}\limits_{t \in [T]} \sigma_t  \cdot \sqrt{2 \log(2T)} \quad,    
    \end{array}
\end{equation}
where step (a) uses $\mathop{\max}\limits_{t \in [T]}|w_{t}| \le \mathop{\max}\limits_{t \in [B]}|w_{t}| + \mathop{\max}\limits_{t \in [T]\setminus[B]}|w_{t}|$ and step (b) uses Fact~\ref{fact: maximal}.

We have
\begin{equation}
    \begin{array}{lll}
       C_{3,T} & = &  \left(\E \left[\mathop{\max}\limits_{t \in [T]     } \frac{|w_t|}{C_{1,t}} \right] + \mathop{\max}\limits_{t \in [T]     } \E \left[\frac{|w_t|}{C_{1,t}} \right]\right) \\
       & \le &  \left(\E \left[\mathop{\max}\limits_{t \in [T]     } \frac{|w_t|}{\rho_t} \right] + \mathop{\max}\limits_{t \in [T]     } \E \left[\frac{|w_t|}{\rho_t} \right]\right)  \\
       & \le & \frac{1}{\mathop{\min}\limits_{t \in [T]} \rho_t} \cdot C_{2,T} + \frac{1}{\mathop{\min}\limits_{t \in [T]} \rho_t}  \cdot \mathop{\max}\limits_{t \in [T]} \sigma_t \sqrt{2 \log 2} \\
       & \le & \frac{1}{\mathop{\min}\limits_{t \in [T]} \rho_t}  \left(1+  \mathop{\max}\limits_{t \in [T]} \sigma_t  \cdot \left(\sqrt{2 \log(2T)} + \sqrt{2 \log 2} \right) \right)\quad.
    \end{array}
\end{equation}
   Plugging in the upper bounds for $C_{2,T}$ and $C_{3,T}$ into Theorem~\ref{theorem: expected regret} concludes the proof. \end{proof}

\end{document}